%% file: ms.tex
\documentclass[twoside,11pt]{article}

\usepackage[margin=1in]{geometry}
\usepackage{authblk}
\usepackage[round]{natbib}
\usepackage{url}
\usepackage{hyperref}

\usepackage{amsthm}
 
\newtheorem{theorem}{Theorem}
\newtheorem{lemma}[theorem]{Lemma} 
\newtheorem{proposition}[theorem]{Proposition} 

\newtheorem{corollary}[theorem]{Corollary}
\newtheorem{definition}[theorem]{Definition}

\input{preamble}

\begin{document}

\title{Monotonic Risk Relationships under Distribution Shifts \\for Regularized Risk Minimization}
\date{} % is on the side for arXiv

\author[1]{Daniel LeJeune\thanks{daniel@dlej.net}}
\author[2]{Jiayu Liu\thanks{jiayu.liu@tum.de}}
\author[2]{Reinhard Heckel\thanks{reinhard.heckel@tum.de}}
\affil[1]{Department of Statistics, Stanford University}
\affil[2]{Department of Electrical and Computer Engineering, Technical University of Munich}

\maketitle

\vspace{-1.3cm}

\input{abstract}

\input{intro}

\input{linear_shift_in_linear_regression_in_finite_dimension}

\input{asymptotic}

\input{linear_shift_in_inverse_problem}

\input{discussion}

\section*{Acknowledgments}
This work was supported by the Institute of Advanced Studies at the Technical University of Munich, the Deutsche Forschungsgemeinschaft (DFG, German Research Foundation) --
456465471, 464123524, and the authors also received funding by the German Federal Ministry of Education and Research and the Bavarian State Ministry for Science and the Arts.
DL was also partially supported by 
the Office of Naval Research grant N00014-18-1-2047 and MURI grant N00014-20-1-2787 and the Army Research Office grant 2003514594.

RH would like to thank Ludwig Schmidt for helpful discussions on linear fits, and on feedback on the manuscript.

\appendix

\input{appendix-proofs}

\newpage

\input{citations.bbl}
\end{document}

%% file: preamble.tex
\usepackage{enumitem}
\usepackage{defs}

\usepackage{tikz}
\usepackage{pgfplotstable}

\usepackage{caption}

\newtheorem{innerassumption}{Assumption}
\newenvironment{assumption}[1]
  {\renewcommand\theinnerassumption{#1}\innerassumption}
  {\endinnerassumption}
  
\definecolor{harvardcrimson}{rgb}{0.79, 0.0, 0.09}
\definecolor{brickred}{rgb}{0.8, 0.25, 0.33}
\definecolor{vividcerise}{rgb}{0.85, 0.11, 0.51}
\definecolor{tangerine}{rgb}{0.95, 0.52, 0.0}
\definecolor{olivine}{rgb}{0.6, 0.73, 0.45}
\definecolor{pistachio}{rgb}{0.58, 0.77, 0.45}
\definecolor{darkseagreen}{rgb}{0.56, 0.74, 0.56}
\definecolor{forestgreen(web)}{rgb}{0.13, 0.55, 0.13}
\definecolor{sacramentostategreen}{rgb}{0.0, 0.34, 0.25}
\definecolor{steelblue}{rgb}{0.27, 0.51, 0.71}
\definecolor{veronica}{rgb}{0.63, 0.36, 0.94}

\usepackage{autonum}

\usepackage[title]{appendix}

%% file: abstract.tex
\begin{abstract}%
Machine learning systems are often applied to data that is drawn from a different distribution than the training distribution. 
Recent work has shown that for a variety of classification and signal reconstruction problems, the out-of-distribution performance is strongly linearly correlated with the in-distribution performance. 
If this relationship or more generally a monotonic one holds, it has important consequences. 
For example, it allows to optimize performance on one distribution as a proxy for performance on the other. 
In this paper, we study conditions under which a monotonic relationship between the performances of a model on two distributions is expected.
We prove an exact asymptotic linear relation for squared error and a monotonic relation for misclassification error for ridge-regularized general linear models under covariate shift, as well as an approximate linear relation for linear inverse problems.
\end{abstract}

%% file: intro.tex
\section{Introduction}
\label{sec:introduction}

Machine learning models are typically evaluated by shuffling a set of labeled data, splitting it into training and test sets, and evaluating the model trained on the training set on the test set. This measures how well the model performs on the distribution the model was trained on. However, in practice a model is most commonly not applied to such in-distribution data, but rather to out-of-distribution data that is almost always at least slightly different. In order to understand the performance of machine learning methods in practice, it is therefore important to understand how out-of-distribution performance relates to in-distribution performance. 

While there are settings in which models with similar in-distribution performance have different out-of-distribution performance~\citep{mccoy19}, a series of recent empirical studies have shown that often, the in-distribution and out-of-distribution performances of models are strongly correlated:

\begin{itemize}
\item \citet{recht19}, \citet{yadav19}, and \citet{miller20} constructed new test sets for the popular CIFAR-10, ImageNet, and MNIST image classification problems and for the SQuAD question answering datasets by following the original data collection and labeling process as closely as possible. For CIFAR-10 and ImageNet the performance drops significantly when evaluated on the new test set, indicating that even when following the original data collection and labeling process, a significant distribution shift can occur. In addition, for all four distribution shifts, the in- and out-of-distribution errors are strongly linearly correlated. 

\item \citet{miller21} identified a strong linear correlation of the performance of image classifiers for a variety of natural distribution shifts. 
Apart from classification, the linear performance relationship phenomenon is also observed in machine learning tasks where models produce real-valued output, for example in pose estimation~\citep{miller21} and object detection~\citep{caine21}.

\item \citet{darestani21} identified a strong linear correlation of the performance of image reconstruction methods for a variety of natural distribution shifts. This relation between in- and out-of-distribution performances persisted for image reconstruction methods that are only tuned, i.e., only a small set of hyperparameters is chosen based on hyperparameter optimization on the training data. %
\end{itemize}

An important consequence of a linear, or more generally, a monotonic relationship between in- and out-of-distribution performances is that a model that performs better in-distribution also performs better on out-of-distribution data, and thus measuring in-distribution performance can serve as a proxy for tuning and comparing different models for application on out-of-distribution data. 

It is therefore important to understand when a linear or more generally a monotonic relationship between the performance on two distributions occurs. 
In this paper we study this question theoretically and empirically for a class of distribution shifts where the feature or signal models come from different distributions, also known as covariate shift.
    
    First, we show that for a real-world regression problem, in- and out-of-distribution performances are linearly correlated. Specifically, we show that for object detection, the performance of models trained on the COCO 2017 training set and evaluated on the COCO 2017 validation set is linearly correlated with the performance on the VOC 2012 dataset. This finding establishes that a linear risk relation also occurs for regression problems, beyond classification problems as established before.%
    
    We then consider a simple linear regression model with a feature vector drawn from a different subspace for in- and out-of-distribution data. We provide sufficient conditions for a linear estimator that characterizes when a linear relation between in- and out-of-distribution occurs. 
    
    Next, we consider a general setup encompassing classification and regression, and consider a distribution shift model on the feature vectors. We consider a large class of estimators obtained with regularized empirical risk minimization, and show that as various training parameters change, including for example the regularization strength or the number of training examples (resulting in different estimators), the relationship between in- and out-of-distribution performances is monotonic. 
    Different classes of estimators follow different monotonic relations, and we also observe this in practice (see Figure~\ref{fig:classification-risk-shifts}). Interestingly, for a certain class of shifts in classification, we recover a linear relation for a nonlinear function of the risks that is remarkably similar to that demonstrated empirically by \mbox{\citet{miller21}}.
    
    Finally, we study linear inverse problems, to understand when a linear relation occurs in a signal reconstruction problem. %
    We consider a distribution shift model consisting of a shift in subspace as well as noise variance, and again characterize conditions under which a linear or near-linear relation between in- and out-of-distribution performances exists.

Our results suggest that linear risk relationships observed in regression and classification actually arise by independent mechanisms, being based on a shift in feature subspace for regression and a shift in feature scaling for classification.

Code for the experiments and figures in this paper can be found at \url{https://github.com/MLI-lab/monotonic_risk_relationships}

\subsection{Prior Theoretical Work on Characterizing Linear Performance Relations}
\label{sec:prior_theoretical_work}

Classical theory for characterizing out-of-distribution performance ensures that the difference between in- and out-of-distribution performance of an estimator is bounded by a function of the distance of the training and test distributions~\citep{quinonero08, ben-david10, cortes14}. 
Such bounds often apply to a class of target distributions. In contrast, we are interested in precise relationships between a fixed source and target distribution. 

Regarding characterizing linear relationships, \citet[Sec.~7]{miller21} proved that for a distribution shift for a binary mixture model, the in- and out-of-distribution accuracies have a linear relation if the features vectors are sufficiently high-dimensional. 
\citet{mania20} showed that an approximate linear relationship occurs under a model similarity assumption that high accuracy models correctly classify most of the data points that are correctly classified by lower accuracy models. %

Most related to our work is that of \citet{tripuraneni21}, who revealed an exact linear relation for squared error of a linear random feature regression model under a covariate shift in the high-dimensional limit. This covariate shift is philosophically similar to the simplifying assumption we make for the main statement and interpretation of our results, and yields a similar linear relation for squared error. However, our results apply to a broader class of general linear models and extend to misclassification error, and we go further to capture how the distribution shift can depend on the task itself, which captures how classification problems can become easier or harder. Moreover, our results predict general monotonic relationships as opposed to only linear ones.

%% file: linear_shift_in_linear_regression_in_finite_dimension.tex
\section{Linear Relations in Regression and Motivation for the Subspace Model}
\label{sec:linear_shift_in_regression_and_motivation_for_subspace_model}

\begin{figure}[t]
\centering
\setlength\tabcolsep{1.5pt} %
\begin{tabular}{ccc}
\begin{tikzpicture}
\begin{axis} [
    xmin=0.005, xmax=0.02, ymin=0.0, ymax=0.08,
    xmin=0.0075, xmax=0.0175, ymin=0.0, ymax=0.08,
    xtick distance=0.005,
    ytick distance=0.02,
    xticklabel style={/pgf/number format/fixed,
                  /pgf/number format/precision=3},
    yticklabel style={/pgf/number format/precision=3,
  /pgf/number format/fixed},
    every tick label/.append style={font=\tiny},
    scaled x ticks = false,
    scaled y ticks = false,
    grid=both,
    minor tick num=1,
    major grid style={lightgray!25},
    minor grid style={lightgray!25},
    width=0.3\textwidth,
    height=0.3\textwidth,
    x label style={at={(axis description cs:0.5,0.1)}, anchor=north},
    y label style={at={(axis description cs:0.22,0.5)},  anchor=south},
    xlabel={COCO 2017 MSE}, %
    ylabel={VOC 2012 MSE}, %
    label style={font=\scriptsize},
    legend entries={$\text{models}$, $\text{linear fit}$, $y=x$},
    legend style={font=\tiny, at={(1.03,1.03)},
        anchor=north west},
]
\addplot [
    mark size=1.2pt, only marks, mark options={solid}, 
    scatter/classes={
        Faster_R-CNN={steelblue},
        Mask_R-CNN={tangerine},
        Keypoint_R-CNN={olivine},
        SSD={forestgreen(web)},
        RetinaNet={veronica},
        YOLOv5={vividcerise}},
    scatter, scatter src=explicit symbolic,
    ]
    table [x=data_p_mean, y=data_q_mean, meta=meta
    ] {./figures/bbox_prediction/average_l2_loss.csv};
\legend{{Faster R-CNN}, {Mask R-CNN}, {Keypoint R-CNN}, RetinaNet, SSD, YOLOv5, linear fit, $y = x$}
\addplot+ [no markers, harvardcrimson, semithick, domain=0.0:0.02] {3.112849*x + 0.012949};
\addplot+ [no markers, black, semithick, dashed, domain=0.0:0.02] {x};
\end{axis}
\end{tikzpicture}
& 
\begin{tikzpicture}
\begin{semilogyaxis} [
    xmin=-10, xmax=522, ymin=0, ymax=360,
    xtick distance=100,
    yticklabel style={/pgf/number format/precision=3,
  /pgf/number format/fixed},
    every tick label/.append style={font=\tiny},
    xmajorgrids, ymajorgrids,
    major grid style={lightgray!25},
    width=0.3\textwidth,
    height=0.3\textwidth,
    x label style={at={(axis description cs:0.5,0.1)}, anchor=north},
    y label style={at={(axis description cs:0.25,0.5)},  anchor=south},
    xlabel={Index}, 
    ylabel={Singular value},
    label style={font=\scriptsize},
    legend entries={$\text{COCO 2017}$, $\text{VOC 2012}$},
    legend style={font=\tiny},
]
\addplot [
    no marks, steelblue, semithick, mark options={solid, steelblue, fill=steelblue}]
    table [x=index, y=singular_value_p] {./figures/bbox_prediction/subspace_analysis.csv};
\addplot [
    no marks, darkseagreen, semithick, mark options={solid, darkseagreen, fill=darkseagreen}]
    table [x=index, y=singular_value_q] {./figures/bbox_prediction/subspace_analysis.csv};
\end{semilogyaxis}
\end{tikzpicture}
& 
\begin{tikzpicture}
\begin{axis} [
    xmin=-10, xmax=522, ymin=0.65, ymax=1.0,
    xtick distance=100,
    ytick distance=0.1,
    yticklabel style={/pgf/number format/precision=3,
  /pgf/number format/fixed},
    every tick label/.append style={font=\tiny},
    grid=both,
    major grid style={lightgray!25},
    width=0.3\textwidth,
    height=0.3\textwidth,
    x label style={at={(axis description cs:0.5,0.1)}, anchor=north},
    y label style={at={(axis description cs:0.28,0.5)},  anchor=south},
    xlabel={$k$}, 
    ylabel={Subspace similarity},
    label style={font=\scriptsize},
]
\addplot [
    no marks, steelblue, semithick, mark options={solid, steelblue, fill=steelblue}]
    table [x=index, y=subspace_similarity] {./figures/bbox_prediction/subspace_analysis.csv};
\end{axis}
\end{tikzpicture}
\end{tabular}
\caption{
Bounding box prediction on COCO 2017 and VOC 2012 datasets. \textbf{Left:} There is an approximate linear relation of mean squared error (MSE) for models trained COCO 2017. \textbf{Middle:} The spectrum of the feature spaces of YOLOv5 on the two datasets decays quickly, which suggests that a feature subspace model could be a reasonable approximation. \textbf{Right:} A principal-angle-based similarity between subspaces spanned by the top $k$ principal components on the two datasets. The subspaces are well-aligned, which is a sufficient condition for a linear relation as stated in Theorem~\ref{thm:sufficient_conditions_for_linear_shift_in_regression}. 
}
\label{fig:bounding_box_prediction}
\end{figure}
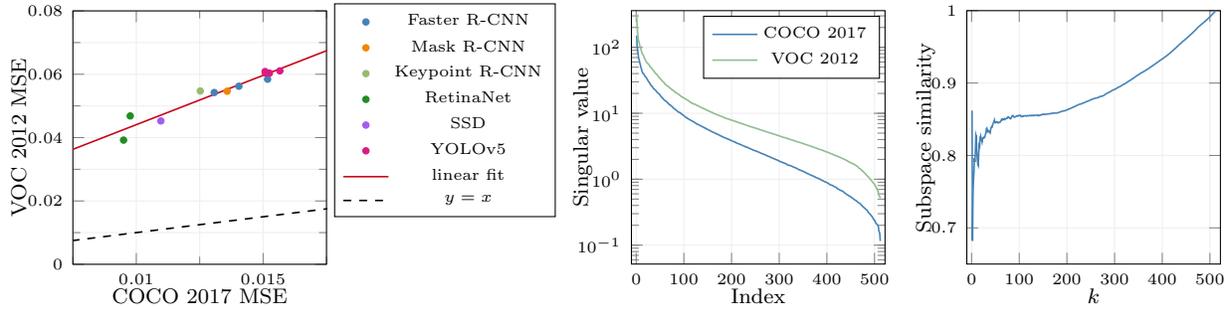

Prior work in the distribution shift literature for prediction tasks has focused on either classification or on problems with real-valued outputs but using discrete performance metrics (for example,  pose estimation~\citep{miller21} and object detection~\citep{caine21}). %
Here, we consider a real-valued squared error metric and show that linear relationships between in- and out-of-distribution performances also occur in a standard regression setup.

We evaluate a collection of neural network models for object detection, which are trained on the COCO 2017 training set~\citep{lin14}: Faster R-CNN~\citep{ren15}, Mask R-CNN~\citep{he17}, Keypoint R-CNN~\citep{he17}, SSD~\citep{liu16}, RetinaNet~\citep{lin17}, and YOLOv5~\citep{redmon16, jocher20}. 
See Figure~\ref{fig:bounding_box_prediction}~(left), where we compute their mean squared errors for bounding box coordinate prediction on the COCO 2017 validation set and the VOC 2012 training/validation set ~\citep{everingham10}.
The  models we consider all perform worse on the out-of-distribution data, and the in- and out-of-distribution performances are approximately linearly related.

It is in general difficult to model distribution shifts analytically. In this work, one aspect of distribution shifts that we model is the change in the subspaces where the feature vectors lie.
To motivate this model, we next examine the feature space of the YOLOv5 model on the in- and out-of-distribution data.

The YOLOv5 model, and all other models considered, can be viewed to make a prediction for an image by generating features through several layers, and then aggregating those features with a linear layer (or a very shallow neural network) to make a prediction. 
We consider the $512$-dimensional feature vectors from the penultimate layer of YOLOv5 as the features. 
Let 
$\{\vz_j^{(i)} \in \reals^{512}: i \in [N_{\text{in}}], j \in [K_{\text{in}}^{(i)}]\}$ and 
$\{\vz_j^{(i)} \in \reals^{512}: i \in [N_{\text{out}}], j \in [K_{\text{out}}^{(i)}]\}$ be the set of feature vectors of the in- and out-of-distribution data, respectively, where $\vz_j^{(i)}$ is the feature vector of the $j$\textsuperscript{th} true positive prediction on image $i$, $N_{\text{in}}$ and $N_{\text{out}}$ are the numbers of images of the respective datasets, and $K_{\text{in}}^{(i)}$ and $K_{\text{out}}^{(i)}$ are the number of true positive predictions on the $i$\textsuperscript{th} images of the respective datasets. 
We perform principal component analysis on these two sets of feature vectors and plot the spectra in Figure~\ref{fig:bounding_box_prediction}~(middle). 
We observe that approximating the feature space by the top $100$ principal components explains $96.0\%$ and $95.6\%$ of the variances of COCO 2017 and VOC 2012 respectively. This observation demonstrates that the feature vectors approximately lie in low-dimensional subspaces. 

Moreover, Figure~\ref{fig:bounding_box_prediction}~(right) shows that the feature subspaces for the two distributions are  overlapping substantially. 
Specifically, Figure~\ref{fig:bounding_box_prediction}~(right) shows the subspace similarity defined as $\sqrt{\norm[2]{\cos(\vtheta)}^2 / k}$ \citep{soltanolkotabi12, heckel15}, where $\vtheta$ is the vector of principal angles between the subspaces spanned by the top $k$ principal components of the individual feature vector sets. %
The subspaces spanned by the top $100$ principal components, which account for over $95\%$ of the variance, have a $0.855$ subspace similarity (note that the maximum value $1$ is achieved for $\vtheta = \vzero$).
More details on the experiment are in Appendix~\ref{sec:experiment_details_for_object_detection}.

Because the output of neural networks is simply a linear model applied to this feature space,
this observation suggests that the relationship between in- and out-of-distribution performances of even highly nonlinear models such as neural networks on data from highly nonlinear spaces may be modeled by a change in linear subspaces of a transformed feature space.  
Therefore, we theoretically study the effect of changes of subspace in linear models and the resulting performance relationships. Our results consider fixed feature spaces, while different deep learning models have different feature representations at the final layer. However, our study can shed light on performance changes of models from the same family that share similar feature representations under distribution shifts.

\section{Linear Relations in Regression in Finite Dimensions}
\label{sec:linear_shift_in_linear_regression_in_finite_dimension}

We begin our theoretical study by considering the linear regression setting under additive noise:
$ %
    y = \vx^\T \vbeta^* + z,
$ %
where $\vbeta^* \in \reals^d$ is a fixed parameter vector that determines the model, and $z$ is independent observation noise. We assume that the feature vector $\vx$ is drawn randomly from a subspace\editedinline{, also known as the hidden manifold model \citep{goldt2020hidden}}. 
\editedinline{Let $d_P, d_Q \le d$.}
For data from distribution $P$, the feature vector is given by \editedinline{$\vx = \mU_P \vc_P$}, where $\mU_P \in \reals^{d \times d_P}$ has orthonormal columns and \editedinline{$\vc_P \in \reals^{d_P}$} is \editedinline{zero-mean and has identity covariance}. The noise variable is zero-mean and has variance  $\sigma_P^2$. 
The data from distribution $Q$ is generated in the same manner, but the signal is from a different subspace with \editedinline{$\vx = \mU_Q \vc_Q$}, where $\mU_Q \in \reals^{d \times d_Q}$ has orthonormal columns,
\editedinline{$\vc_Q \in \reals^{d_Q}$ is zero-mean and has identity covariance},
and the noise is zero-mean and has variance $\sigma_Q^2$. 

For an estimate $\widehat{\vbeta}$ of $\vbeta^*$, define the risk on distribution $P$ with respect to the squared error metric as $\risk_P(\widehat{\vbeta}) = \EX[\vx \sim P]{(y - \vx^\transp \widehat{\vbeta})^2}$ (respectively $\risk_Q(\widehat{\vbeta})$ on distribution $Q$).
We are interested in the relation of those risks for a class of estimators. We consider an estimate of the model parameter $\vbeta^*$ assuming knowledge of the distribution for simplicity, equivalent to having large amounts of training data. The analysis can be extended readily to estimates based on finite samples. We consider the estimator
\begin{align}
    \widehat{\vbeta}_\lambda = \argmin_{\vbeta} \EX[P]{(\vbeta^\T \vx - y)^2} + \lambda \norm[2]{\vbeta}^2,
\end{align}
parameterized by the regularization parameter $\lambda$. 
It can be shown that 
$
    \widehat{\vbeta}_\lambda = \alpha \mU_P \mU_P^\T\vbeta^*
$ for $\alpha = 1 / (1 + \lambda)$, 
which is the projection of $\vbeta^*$ onto the subspace scaled by the factor $\alpha \in [0,1]$.

The following theorem provides sufficient conditions for a linear relation between the in- and out-of-distribution risks $\risk_P(\widehat{\vbeta}_\lambda)$ and $\risk_Q(\widehat{\vbeta}_\lambda)$ of this class of estimators parameterized by the regularization parameter $\lambda$. 
Theorem~\ref{thm:sufficient_conditions_for_linear_shift_in_regression} is a consequence of Theorem~\ref{thm:sufficient_necessary_conditions_for_linear_shift_in_regression} in Appendix~\ref{sec:proof_of_conditions_for_linear_shift_in_regression}, which also provides a necessary condition for a linear risk relation. 
\begin{theorem}[Sufficient conditions]
The out-of-distrubiton risk $\risk_Q(\widehat{\vbeta}_\lambda)$ is an affine function of the in-distribution risk  $\risk_P(\widehat{\vbeta}_\lambda)$ as a function of the regularization parameter $\lambda$ 
if one of the following conditions holds: 
\begin{enumerate}[label=(\alph*), ref=(\alph*)]
    \item $\text{range}(\mU_Q) \subseteq \text{range}(\mU_P)$, or %
    $\text{range}(\mU_P) \subseteq \text{range}(\mU_Q)$;
    \label{cond:subspace-align}
    
    \item $\vbeta^* \in \text{range}(\mU_P)$.
    \label{cond:beta-eigval}
\end{enumerate}
Moreover, for random $\vbeta^*$, the expected out-of-distribution risk, $\EX[\vbeta^*]{\risk_Q(\widehat{\vbeta}_\lambda)}$, is an affine function of the expected in-distribution risk $\EX[\vbeta^*]{\risk_P(\widehat{\vbeta}_\lambda)}$ if
\begin{edited}
\begin{enumerate}[label=(\alph*), ref=(\alph*)]
    \setcounter{enumi}{2}
    \item $\EX[]{\vbeta^* \vbeta^{*\T}} = \mI$. 
    \label{cond:random-beta}
\end{enumerate}
\end{edited}

\label{thm:sufficient_conditions_for_linear_shift_in_regression}
\end{theorem}

\editedinline{Condition~\ref{cond:subspace-align} is a property of the distribution shift itself.} When the subspaces are aligned between the two distributions, we observe a linear risk relationship for the set of estimators parameterized by $\lambda$. 
Recall from the previous section, that the feature subspaces of the object detection model we evaluate roughly align, as shown in Figure~\ref{fig:bounding_box_prediction}~(right). 
Thus, our theorem suggests a linear relationship, which in turn sheds light on the linear relationship we observed in practice. 
We remark that the linear relationship guaranteed by Theorem~\ref{thm:sufficient_conditions_for_linear_shift_in_regression} is exact assuming full knowledge of the source distribution, but only approximate in the finite sample regime for an estimate that minimizes the regularized empirical risk.

\begin{edited}
Condition~\ref{cond:beta-eigval} is a property of the parameter vector $\vbeta^*$ and its learnability under distribution $P$. Under condition~\ref{cond:beta-eigval}, $\widehat{\vbeta}_\lambda = \alpha \vbeta^*$, which greatly simplifies the risks:
\begin{align}
    \risk_P(\widehat{\vbeta}_\lambda) = (1 - \alpha)^2 \vbeta^{*\transp} \vbeta^* + \sigma_P^2
    \quad
    \text{and}
    \quad
    \risk_Q(\widehat{\vbeta}_\lambda) = (1 - \alpha)^2 \vbeta^{*\transp} \mU_Q \mU_Q^\transp \vbeta^* + \sigma_Q^2.
\end{align}
It is thus very clear that there is a monotonic relation, as both are affine in $(1 - \alpha)^2$.

Condition~\ref{cond:random-beta} meanwhile is a property of randomness in $\vbeta^*$ that leads to the elimination of interaction terms that would prevent a monotonic relation. While the above result is given for the expectation, the same effect would also occur for single problem instances in high dimensions due to concentration of measure.
 
The intuition behind these three conditions all carry over to our more general results.
\end{edited}

%% file: asymptotic.tex
\section{Asymptotic Monotonic Relations for General Linear Models}
\label{sec:asymptotic}

In the previous section, we demonstrated a linear risk relationship under a subspace shift for linear regression models. In this section, we provide a much more general result that holds for a larger class of distribution shifts, setups (i.e., regression and classification), and estimators, specifically for a class of estimators based on regularized empirical risk minimization.

\subsection{Linear Model Framework}

We consider a general framework of linear models $f(\vx) = \phi(\vx^\transp \vbeta)$ for some $\vbeta \in \reals^d$, labeling function $\phi \colon \reals \to \reals$, and centralized Gaussian data under two distributions
\begin{align}
P\colon \vx \sim \mc N(\vzero, \tfrac{1}{d} \mSigma_P) \quad\text{and}\quad
Q\colon \vx \sim \mc N(\vzero, \tfrac{1}{d} \mSigma_Q),
\end{align}
where $\mSigma_P$ and $\mSigma_Q$ are positive semidefinite covariance matrices. Given a ground truth model $f^*(\vx) = \phi(\vx^\transp \vbeta^*)$, the \emph{risk} of a model $f(\vx) = \phi(\vx^\transp \vbeta)$ with respect to an error metric $\psi \colon \reals^2 \to \reals$ on distributions $P$ and $Q$ as
\begin{align}
    \risk_P(\vbeta) \defeq \expect[\vx \sim P]{\psi(\vx^\transp \vbeta^*, \vx^\transp \vbeta)}
    \quad \text{and} \quad
    \risk_Q(\vbeta) \defeq \expect[\vx \sim Q]{\psi(\vx^\transp \vbeta^*, \vx^\transp \vbeta)}.
\end{align}
We consider the squared error $\psi(z^*, z) = (z^* - z)^2$ and misclassification error $\psi(z^*, z) = \ind \set{z^* z < 0}$ as error metrics for regression and classification, respectively. Now define the random variables, often referred to as the decision functions,
\begin{align}
    (Z_P^*, Z_P) = (\vx^\transp \vbeta^*, \vx^\transp \vbeta) : \vx \sim P
    \quad \text{and} \quad
    (Z_Q^*, Z_Q) = (\vx^\transp \vbeta^*, \vx^\transp \vbeta) : \vx \sim Q.
\end{align}
As we capture in the following proposition, the risks for any linear model $f(\vx) = \phi(\vx^\transp \vbeta)$ depend only on a few parameters defining the covariances of the decision functions.
\begin{proposition}
    The vectors $(Z_P^*, Z_P)$ and $(Z_Q^*, Z_Q)$ are zero-mean bivariate normal random vectors.
    Furthermore, $\risk_P(\vbeta)$ and $\risk_Q(\vbeta)$ are functions only of the covariance matrices $\mathrm{Cov}(Z_P^*, Z_P) \in \reals^{2\times 2}$ and $\mathrm{Cov}(Z_Q^*, Z_Q) \in \reals^{2\times 2}$, respectively.
\end{proposition}

Thus, while the covariance matrices $\mSigma_P$, $\mSigma_Q$, and the model parameters $\vbeta^*$ and $\vbeta$ in general comprise on the order of $d^2$ parameters, the risks $\risk_P(\vbeta)$ and $\risk_Q(\vbeta)$ are characterized by no more than 6  parameters of the covariance matrices $\mathrm{Cov}(Z_P^*, Z_P)$ and $\mathrm{Cov}(Z_Q^*, Z_Q)$. 
In order to have a monotonic relation between the risks $\risk_P(\vbeta)$ and $\risk_Q(\vbeta)$ the dependency needs to be reduced to a single parameter, which requires additional assumptions on the class of models and the distribution shifts, which we state in the next subsection.

\subsection{Asympototic Estimation with Regularized Empirical Risk Minimzation}

We consider predictors $\hat{f} = \phi(\vx^\transp \widehat{\vbeta})$ where the parameter $\widehat{\vbeta}$ is the ridge-regularized empirical risk minimization (ERM) estimate 
\begin{align}
    \label{eq:erm-ridge}
    \widehat{\vbeta}(\setD, \ell, \lambda) = \argmin_\vbeta \sum_{i=1}^n \ell(y_i, \vx_i^\transp \vbeta) + \frac{\lambda}{2} \norm[2]{\vbeta}^2,
\end{align}
where $\setD = \{(\vx_1, y_1), \ldots, (\vx_n,y_n)\}$ is  a training set \editedinline{with covariates $\vx_i \sim P$}, $\ell \colon \reals^2 \to \reals$ a loss function, and $\lambda > 0$ a regularization parameter.

In finite dimensions, determining the in- and out-of distribution risk via determining the covariances $\mathrm{Cov}(Z_P^*, Z_P)$ and $\mathrm{Cov}(Z_Q^*, Z_Q)$ even for linear models with convex loss functions is not possible in general, making the task of identifying a monotonic risk relation difficult. Fortunately, however, it has recently been shown~\citep{thrampoulidis18, emami20, loureiro21} that as the problem dimensionality becomes large, thanks to concentration of measure effects, the solution to regularized ERM problems can be characterized by the solution of a system of scalar fixed point equations in only a few variables. Our result relies on such an asymptotic characterization by~\citet{loureiro21}.

In the following, we state the asymptotic setup, data generation process, and distribution-shift model that we consider as an assumption, so that we can refer to it later.

\begin{assumption}{A}[Setup]
    \label{assump:asymp-cgmt}
    \hfill    
    \begin{enumerate}[label=\normalfont{(A\arabic*)}, ref=A\arabic*]
        \item {\bf Asymptotically proportional regime}.  The training data set size $n$ and dimensionality $d$ converge with fixed finite ratio $d / n$.
        \label{assump:asymp-cgmt:proportional}

        \item {\bf Training data generation.} The training data $\setD_n = \{(\vx_1,y_1),\ldots, (\vx_n,y_n)\}$ is from distribution $P$, and independently generated as $\vx_i \overset{\iid}{\sim} \normal(\vzero, \tfrac{1}{d} \mSigma_P)$ and $y_i = \varphi(\vx_i^\transp \vbeta^*, \xi_i)$ for a labeling function $\varphi \colon \reals^2 \to \reals$, random noise $\xi_i$ independent of $\vx_i$ and ground truth coefficient vector $\vbeta^*$. Additionally, $\lim_{n \to \infty} \tfrac{1}{n} \expect{\norm[2]{\vy}^2} < \infty$.
        \label{assump:asymp-cgmt:data}

        \item {\bf Ground-truth coefficient vector and structure of the covariances.} The ground truth coefficient vector $\vbeta^*$ has elements drawn i.i.d.\ from a zero-mean sub-Gaussian distribution with variance $\sigma_\beta^2$ and is independent of $\setD$, and      $\mSigma_P = \mPi_P$ is a projection operator onto a subspace of dimension $d_P$ such that $d_P / d \to r_P \in (0, 1]$. Furthermore, the covariances $\mSigma_P$ and $\mSigma_Q$ are simultaneously diagonalizable.
        \label{assump:asymp-cgmt:gt-cov}

        \item {\bf Loss function of ERM.} The loss function $\ell$ is a proper, lower semi-continuous, convex function that is pseudo-Lipschitz of order 2 (see Definition~\ref{def:pl-continuous} in Appendix~\ref{sec:thm:monotonic-risks:proof} for a formal definition) such that for all $n$ and $c > 0$, 
        if $\norm[2]{\vz} \leq c \sqrt{n}$ then there exists a positive constant $C$ such that $\sup_{\vz' \in \partial_\vz \bar{\ell}(\vy, \vz)} \norm[2]{\vz'} \leq C \sqrt{n}$, where $\bar{\ell}(\vy, \vz) = \sum_{i=1}^n \ell(y_i, z_i)$.
        Furthermore, for the standard normal random vector $\vg \in \reals^n$, $\tfrac{1}{n} \expect{\bar{\ell}(\vy, \vg)}$ is uniformly bounded in $n$. 
        \label{assump:asymp-cgmt:loss}
    \end{enumerate}
\end{assumption}

The data generating process (\ref{assump:asymp-cgmt:data}) and the assumption on the loss function of ERM (\ref{assump:asymp-cgmt:loss}) are standard for most convex and linear ERM formulations used in machine learning for regression and classification. 

The assumptions on the ground truth coefficients and covariance matrices in Assumption \ref{assump:asymp-cgmt:gt-cov} are stronger than necessary;  our result can in fact even be proved for deterministic $\vbeta^*$ and essentially arbitrary $\mSigma_Q$ and non-isotropic $\mSigma_P$ (see Appendix~\ref{sec:thm:monotonic-risks:proof}). However, these assumptions greatly simplify the form of the results at little expense of generality. 

Assumption \ref{assump:asymp-cgmt:proportional} puts us in the proportional asymptotics regime, but the concentration effects are often realized at only modest data sizes; see Figure~\ref{fig:simulation-risk-shifts}. 

Under Assumption~\ref{assump:asymp-cgmt}, \editedinline{the ERM estimator has the form $\widehat{\vbeta} = \mPi_P(a \vbeta^* + c \vg)$ for some $\vg \sim \normal(\vzero, \mI_d)$ independent of $\vbeta^*$ (see Corollary~\ref{cor:asymp-cgmt}), extending the intuition from Theorem~\ref{thm:sufficient_conditions_for_linear_shift_in_regression}. Therefore,}
the covariances $\mathrm{Cov}(Z_P^*, Z_P)$ and $\mathrm{Cov}(Z_Q^*, Z_Q)$ \editedinline{have only two degrees of freedom ($a$ and $c$)} in the asymptotic limit for fixed $P$, $Q$, and $\vbeta^*$, even as we vary a number of different learning problem parameters such as loss function, noise level, labeling function, regularization strength, and number of training examples (see Lemma~\ref{lem:asymp-cgmt-cov}). 
\editedinline{Even with only two degrees of freedom, this is still not enough to imply a monotonic relation for general risks (see Section~\ref{sec:no-monotonic-relation}); however, remarkably, it turns out that this specific structure is sufficient for monotonicity for both squared error and misclasification error.}

For this Setup (\ref{assump:asymp-cgmt}),  %
\begin{edited}
the monotonic relations between in- and out-of-distribution risks
for distributions $P$ and $Q$ and ground truth $\vbeta^*$ are entirely described by only three limiting scalar parameters of the distribution shift, which we define next. Our assumption that these quantities converge is stronger than necessary; we only need that these quantities be almost surely uniformly bounded (e.g., by making $\mSigma_Q$ uniformly bounded in operator norm and have non-vanishing subspace overlap with $\mPi_P$), in which case we can simply apply this assumption and our results to each convergence subsequence. However, to keep the statements of our results clean, we assume convergence of these parameters.%
\end{edited}
\begin{assumption}{B}[Parameters]
    \label{assump:asymp-simple}
    The following limits exist for $\vbeta_P^* \defeq \mPi_P \vbeta^*$ almost surely:
    \begin{gather}
        \gamma = \lim_{d \to \infty} \frac{\vbeta_P^{*\transp} \mSigma_Q \vbeta_P^*}{d_P \sigma_\beta^2},
        \qquad
        \mu = \lim_{d \to \infty} \frac{\vbeta^{*\transp} \mSigma_Q \vbeta^*}{\vbeta_P^{*\transp} \mSigma_Q \vbeta_P^*} \geq 1, 
        \qquad
        \kappa = \lim_{d \to \infty} \frac{ \tr[\mSigma_Q \mPi_P]}{d_P}.
    \end{gather}
\end{assumption}

The parameters $\gamma$ and $\mu$ are straightforward to interpret.
The parameter $\gamma$ captures the ratio of the energy of $\vbeta_P^*$ as measured by $\mSigma_Q$ compared to $\mSigma_P$. If $\mSigma_Q$ is a scaled projection operator $\tau \mPi_Q$ \editedinline{for some $\tau > 0$} with $d_{PQ}$ dimensions overlapping with $\mPi_P$, then $\gamma = \tau d_{PQ} / d_P$. The parameter $\mu$ captures the ratio of the total energy of $\vbeta^*$ as measured by $\mSigma_Q$ compared to its restriction to the subspace determined by $\mPi_P$. For the same scaled projection operator example, if $d_Q$ is the dimension of the subspace of $\mPi_Q$, then $\mu = d_Q / d_{PQ}$.

The parameter $\kappa$ introduces the nuance of \emph{task dependence} of the distribution shift. 
Note that the ground-truth parameter $\vbeta^*$ and the covariance matrix $\mSigma_Q$ might be  statistically correlated. 
(We might like to consider $\mSigma_Q$ to be deterministic, whereas $\vbeta^*$ is a random variable. However, our results in Appendix~\ref{sec:thm:monotonic-risks:proof} hold almost surely for a fixed, deterministic, covariance--ground-truth pair $(\mSigma_Q, \vbeta^*)$; so we can think about this pair as deterministic or correlated). 
As an example of such a correlation, $\mSigma_Q$ may have larger eigenvalues in the directions where $\vbeta^*$ is larger in magnitude, and therefore $\gamma > \kappa$. Intuitively, since we assume $\mSigma_P$ to be isotropic on the subspace, this means that at test time, the prediction depends more on coefficients that were learned better during training, making the problem easier. Conversely, if $\gamma < \kappa$, $\mSigma_Q$ and $\vbeta^*$ are anti-correlated, and the prediction becomes more difficult since features that were learned poorly are emphasized more highly. 
This can be summarized with the ratio $\kappa / \gamma$, which when 
less than 1 implies an easier distribution shift, and when
greater than 1 implies a harder one. When $\gamma = \kappa$, we say the shift is \emph{task-independent}. The case of \emph{task-dependent} shifts where $\kappa \neq \gamma$ cannot be captured by the $\mSigma_Q = \tau \mPi_Q$ we used to explain $\gamma$ and $\mu$, as it does not allow $\mSigma_Q$ and $\vbeta^*$ to be correlated.

\begin{figure}[t]
    \centering
    \includegraphics[width=5.5in]{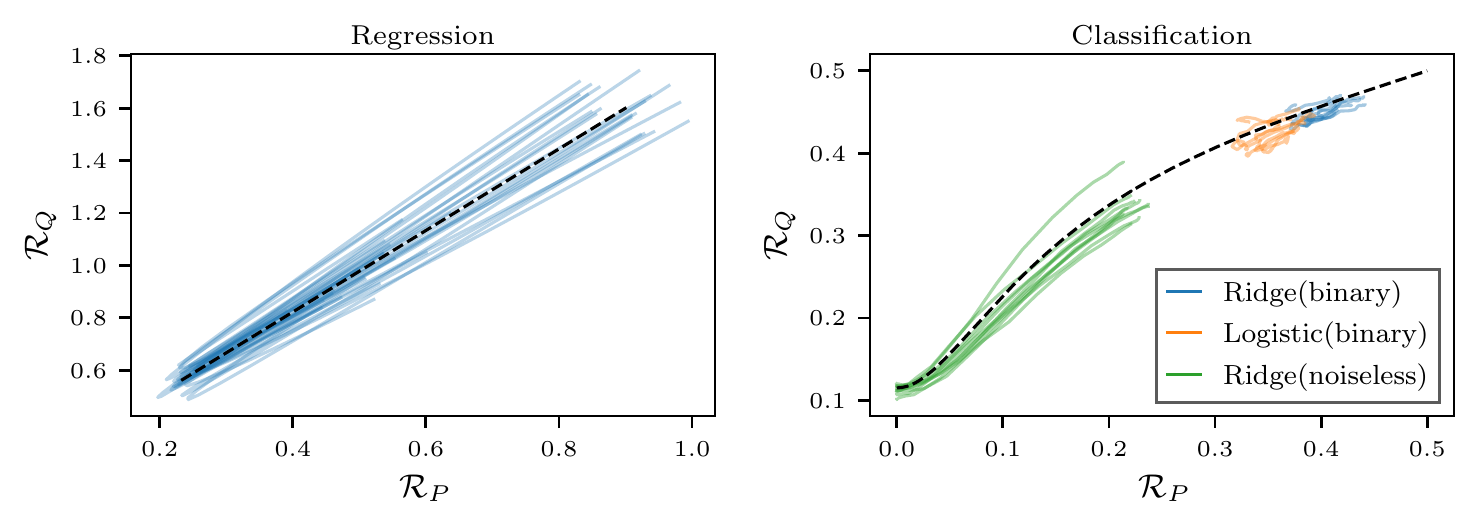}
    \caption{The risk relationships for data generated according to our distribution shift model match our theoretical results (dashed). Each colored curve corresponds to a sweep of the regularization strength of a single model on a single random trial. For both plots, we consider a subspace shift model with $\mSigma_Q = \tau \mPi_Q$ having $d_P/d = 0.9$, $d_Q/d = 0.8$, $d_{PQ}/d = 0.7$, and $\tau = 2$.
    We use $n = 1000$, $d = 800$,  $\sigma_\beta^2 = 1$, and have $\gamma \approx 1.56$, and $\mu \approx 1.14$. 
    \textbf{Left:} Mean squared error for ridge regression models (blue) trained on $y_i = \vx_i^\transp \beta^* + \sigma \xi_i$ for $\sigma^2 = 0.2$ and $\kappa = \gamma$. Although the tuning parameter overshoots the minimizer in the parameter sweep, it still always lies approximately on the line. \textbf{Right:} Misclassification error for ridge regression (blue) and logistic regression (orange) models with ridge penalty trained on corrupted binary labels generated as $\Pr(y_i = \mathrm{sign}(\vx_i^\transp \vbeta^*)) = 0.8$ with $\kappa = 5 \gamma$. We also plot ridge regression trained on noiseless labels $y_i = \vx_i^\transp \vbeta^*$ (green) to illustrate that the result is independent of the labeling function, depending only on the feature distribution shift.}
    \label{fig:simulation-risk-shifts}
\end{figure}

\subsection{Main Result}
\label{sec:mainresoult-assymptotic}

We are now ready to state and discuss our main result. For the proof as well as a more general result without Assumption~\ref{assump:asymp-cgmt:gt-cov} that covers arbitrary deterministic $(\mSigma_P, \mSigma_Q, \vbeta^*)$, see Theorems~\ref{thm:monotonic-risk-squared-error} and \ref{thm:monotonic-risk-classification-error} in Appendix~\ref{sec:thm:monotonic-risks:proof}.

\begin{theorem}[Monotonic risk relations]
    \label{thm:monotonic-risks}
    Under Assumption~\ref{assump:asymp-cgmt}, the following hold with probability 1 in the limit as $d \to \infty$ for all $\widehat{\vbeta} = \widehat{\vbeta}(\setD_n, \ell, \lambda)$ solving \eqref{eq:erm-ridge}.
    \begin{enumerate}[label=\normalfont{(\alph*)}]
    
        \item Regression. For $\psi(z^*, z) = (z^* - z)^2$, there exists a monotonic relation between $\risk_Q(\widehat{\vbeta})$ and $\risk_P(\widehat{\vbeta})$ that depends only on $(P, Q, \vbeta^*)$ if and only if Assumption~\ref{assump:asymp-simple} holds with $\gamma = \kappa$. If this relation exists, it is
        \begin{align}
            \risk_Q(\widehat{\vbeta}) = \gamma \risk_P(\widehat{\vbeta}) + \gamma r_P \sigma_\beta^2 (\mu - 1).
        \end{align}

        \item Classification. For $\psi(z^*, z) = \ind \set{z^* z < 0}$, there exists a monotonic relation between $\risk_Q(\widehat{\vbeta})$ and $\risk_P(\widehat{\vbeta})$ that depends only on $(P, Q, \vbeta^*)$ if and only if Assumption~\ref{assump:asymp-simple} holds.
        If this relation exists, it is
        \begin{align}
            \sec^2 (\pi \risk_Q(\widehat{\vbeta})) = \tfrac{\kappa \mu}{\gamma} \paren{\sec^2(\pi \risk_P(\widehat{\vbeta})) - 1} + \mu,
        \end{align}
        where $\sec(t) = \tfrac{1}{\cos(t)}$. %
        Furthermore, if $\mu = 1$,
        then
        \begin{align}
            \log ( \tan (\pi \risk_Q(\widehat{\vbeta}))) = \log ( \tan (\pi \risk_P(\widehat{\vbeta}))) + \tfrac{1}{2} \log \tfrac{\kappa}{\gamma}.
        \end{align}
    \end{enumerate}
\end{theorem}

Our result states that we have a monotonic relation between in- and out-of-distribution risks under our distribution shift model, for \emph{all estimates} $\widehat{\vbeta}(\setD_n, \ell, \lambda)$ that solve a problem of the form \eqref{eq:erm-ridge}, including, e.g., as we vary the training set size $n$, the regularization parameter $\lambda$, or even the labeling $\varphi$ or loss function $\ell$.

Figure~\ref{fig:simulation-risk-shifts} illustrates this behavior approximately in finite dimensions; there we plot the prediction of our theory along with realizations of data and estimates $\widehat{\vbeta}(\setD_n, \ell, \lambda)$. %
We see effects described by Theorem~\ref{thm:monotonic-risks} in action: two models with the same risk on the distribution that generated the training data have the same risk on the new distribution, regardless of whether they were trained using regression or classification labels, of which particular loss function was used in training, of the training sample size, or of the level of label noise. As we can see in the figure, in finite dimensions individual models can have locally non-monotonic relationships, and it is only as the system becomes asymptotically large and concentration of measure phenomena are realized that the monotonic relation emerges.

For both regression and classification, the risk relations are \emph{linear}, with classification requiring the transformation $\risk \mapsto \sec^2(\pi \risk)$ first before it becomes linear. This linearity is no coincidence; 
as we prove, whenever the risk depends linearly (after a fixed transformation) on some of the parameters of the covariances $\mathrm{Cov}(Z_P^*, Z_P)$ and $\mathrm{Cov}(Z_Q^*, Z_Q)$, as is the case for both squared error and misclassification error, the only monotonic relation that can exist is a linear one. We refer reader to 
Appendix~\ref{sec:thm:monotonic-risks:proof} for proof details.

The risk relations are similar in that for both regression and classification, $\mu > 1$ indicates irreducible error due to a new subspace in $Q$ that was unseen during training on $P$.  
However, the regression and classification risk relations also have a key difference: the squared error risk relation for regression only holds when $\gamma = \kappa$---i.e., only for task-independent shifts. This means that the subspace shift model with $\mSigma_Q = \tau \mPi_Q$ captures all aspects of the regression risk relation. 

The classification risk relation, on the other hand, holds for task-dependent shifts with $\gamma \neq \kappa$. In particular, if we let $\mu \to 1$, then we find that the risk relation is remarkably similar to the empirical observation by \citet{miller21} that the risk relation is linear after applying an inverse Gaussian cumulative distribution function transformation $\Phi^{-1}(\cdot)$. Note that the $\log(\tan(\pi \cdot))$ transformation in Theorem~\ref{thm:monotonic-risks} is strikingly similar to $\Phi^{-1}(\cdot)$; in fact, $\sup_{u \in \reals} |\frac{1}{2}\Phi(u / \sqrt{2}) - \frac{1}{\pi} \tan^{-1}(e^u)| \leq 0.01$.
This suggests that such ``natural'' distribution shifts formed by repeated dataset collection may have no subspace shift component ($\mu \to 1$), but rather only a task-dependent shift ($\gamma \neq \kappa$).
We illustrate the behavior of the classification risk shift for different values of $\mu$ and $\kappa / \gamma$ in Figure~\ref{fig:classification-risk-shifts} (left).

\begin{figure}[t]
    \centering
    \includegraphics[width=\textwidth]{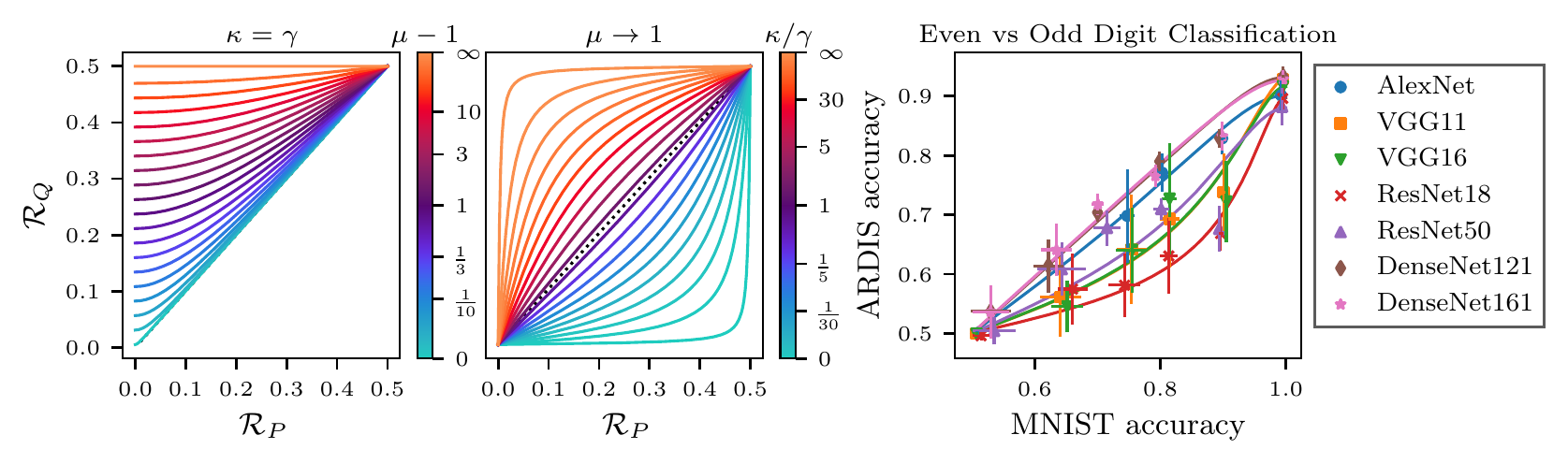}
    \caption{\textbf{Left/middle:} We plot the theoretical risk relation curves for misclassification error. General behavior of the risk relation is a combination of the two behaviors we demonstrate here. \emph{Left:} $\risk_Q$ converges to a nonzero limiting value as $\risk_P \to 0$, which is determined by $\mu$. \emph{Middle:} While keeping the same limiting value of $\risk_Q$ as $\risk_P \to 0$, the shift is harder as $\kappa / \gamma$ gets larger (red) and easier as $\kappa / \gamma$ gets smaller (blue). 
    \textbf{Right:} We train deep network models on classifying even vs. odd handwritten digits from the MNIST and ARDIS datasets, evaluating test performance during training as validation accuracy milestones are reached (dots with errorbars over 8 trials). We also plot our theoretical risk relation with $\mu$ and $\kappa / \gamma$ chosen to minimize squared error of the fit for each model. 
    }
    \label{fig:classification-risk-shifts}
\end{figure}

For different feature spaces, our theory predicts different monotonic relations. This is also observed in practice: in Figure~\ref{fig:classification-risk-shifts} (right), we show that except for the ResNet50 model, our theory predicts well the risk relation as a function of early stopping for deep network models trained on MNIST~\citep{lecun10}, an easy handwritten digit classification task, and applied to ARDIS~\citep{kusetogullari20}, a more difficult handwritten digits dataset. 
See Appendix~\ref{sec:motivation-feature-scaling} for a discussion of how this distribution shift fits the task-dependence shift model.
The fits in Figure~\ref{fig:classification-risk-shifts} show that different neural network models, which have their own respective implicit feature spaces, 
result in different monotonic risk relations.
Because the shift is from an easy task to a hard one, we expect to see similar behavior to the case $\gamma < \kappa$, which matches the general trend of the fits, with some fits tending toward more or less task dependence based on model class.
The tendency of models to dip in performance on ARDIS around 0.9 accuracy on MNIST is, we believe, a result of the change in the learned features of the networks during training, and is worst for ResNet50.

\begin{edited}
\subsection{Settings without Monotonic Relations}
\label{sec:no-monotonic-relation}

Given the generality of the result in Theorem~\ref{thm:monotonic-risks} across essentially any labeling function, training loss, and regularization strength, one might conjecture that the result holds for any risk and for any regularized ERM estimator. This is not the case, however, as the monotonic risk relations only arise due to the special structure of the risks and of ridge regularization.

As mentioned in the previous section, and as we elaborate on in the proof in Appendix~\ref{sec:thm:monotonic-risks:proof}, monotonic risk relations arise when the metric $\psi$ depends linearly on some one-dimensional function of the decision function covariances $\mathrm{Cov}(Z_P^*, Z_P)$ and $\mathrm{Cov}(Z_Q^*, Z_Q)$. The fact that squared error and misclassification error depend on different functions of the covariances is the first clue that the monotonicity of risk relations might not be universal. Indeed, the risk relations that arise are substantially distinct, as the misclassification relation captures task dependent shifts while squared error does not. 

As important counterexamples, popular convex losses used to train classification models such as the hinge loss and logistic loss do not exhibit monotonic risk relations. By Lemma~\ref{lem:asymp-cgmt-cov}, we know that the decision function covariances have only three degrees of freedom $a, b, c$ (for general $\mSigma_P$), but that the monotonic relation should hold regardless of how these are varied. In 
Figure~\ref{fig:logistic-hinge-losses}, however, we show that as we vary even only a single parameter (here $a$), the hinge loss and logistic loss do \emph{not} exhibit monotonic relations, while the misclassification error does. 
In general, monotonicity is further destroyed as we vary more degrees of freedom. This counterexample suggests that practitioners should be careful in their choice of validation metric: optimization of the in-distribution validation loss may not coincide with optimization of the out-of-distribution loss. Choosing a validation metric for which we expect monotonicity, such as misclassification error, is the better choice.

\begin{figure}[t]
    \centering
    \includegraphics[width=5.5in]{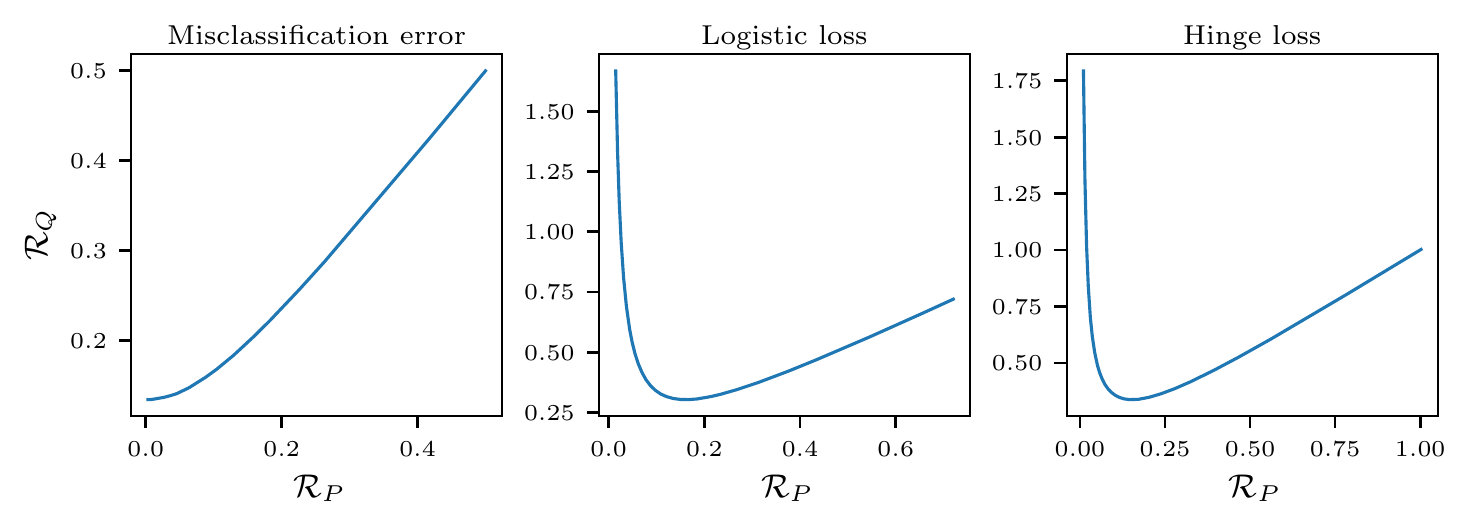}
    \caption{\editedinline{Using Monte Carlo simulation with $10^6$ random draws of $(Z^*, Z)$ from both $P$ and $Q$ under Assumption~\ref{assump:asymp-cgmt}, we compute the risk relationships for misclassification error (\textbf{left}) alongside the logistic loss $\psi(z^*, z) = \log(1 + \exp(-\mathrm{sign}(z^*) z))$ (\textbf{middle}) and the hinge loss $\psi(z^*, z) = \max \set{1 - \mathrm{sign}(z^*) z}$ (\textbf{right}). Here we consider a subspace shift model with $d_P/d = 0.9$, $\sigma_\beta = 1$, $\gamma = 1$, $\kappa = 1$, $\mu = 1.2$. We fix degrees of freedom $b = c = 1$ and vary $a$. Unlike the misclassification error, these losses do not exhibit monotonic risk relationships as a function of $a$.}}
    \label{fig:logistic-hinge-losses}
\end{figure}

Another way that monotonicity can be broken is by changing the dependence of the decision function covariance on the underlying free parameters $a, b, c$. This occurs, for example, if we change the regularizer from the ridge penalty $\tfrac{1}{2}\norm[2]{\cdot}^2$ to some other regularizer such as the $\ell_1$ norm $\norm[1]{\cdot}$. As we show in Appendix~\ref{sec:extend-penalty}, for separable regularizers, we still have monotonic relations, but now for only a restricted class of distributions shifts. Specifically, we only have monotonic relations in the task-independent setting $(\gamma = \kappa)$, since in this case the covariances still admit similar linear decompositions. Otherwise, the nonlinearity due to the regularization penalty destroys monotonicity.

\end{edited}

%% file: linear_shift_in_inverse_problem.tex
\section{Linear Relations in Linear Inverse Problems}
\label{sec:linear_shift_in_linear_inverse_problems}

In this section, we switch to signal reconstruction problems, where linear relationships are also observed for signal reconstruction methods under distribution shifts \citep{darestani21}.

We consider a linear inverse problem setup where the measurement $\vy$ is generated by a linear transform of the signal $\vx$ plus some additive noise $\vz$ independent of $\vx$, i.e., $\vy = \mA\vx + \vz$,
where $\mA \in \mathbbm{R}^{n \times d}$ with $n \le d$, $\vx \in \mathbbm{R}^d$ and $\vz \in \mathbbm{R}^n$. 
We assume a similar signal subspace model as in Section~\ref{sec:linear_shift_in_linear_regression_in_finite_dimension}: for data from distribution $P$, the signal is given by \editedinline{$\vx = \mU_P \vc_P$}, where $\mU_P \in \mathbbm{R}^{d \times d_P}$ has orthonormal columns and $\vc_P \in \reals^{d_P}$ is \editedinline{zero-mean and has identity covariance}. The noise variable $\vz$ is independent of \editedinline{$\vc_P$} and has mean zero and covariance matrix $\sigma_P^2\mI$. The data from distribution $Q$ is generated in the same manner, but the signal is from a different subspace, i.e., \editedinline{$\vx = \mU_Q \vc_Q$}, where $\mU_Q \in \mathbbm{R}^{d \times d_Q}$ is orthonormal, \editedinline{$\vc_Q \in \reals^{d_Q}$ is zero-mean and has identity covariance}, and the covariance matrix of the independent noise $\vz$ is $\sigma_Q^2\mI$. We assume that the number of measurements is larger than the subspace dimension, i.e., \editedinline{$d_P, d_Q \le n$}.

We consider the class of signal estimates given by
\begin{align}
    \widehat{\vx}_\lambda(\vy) = \mW^* \vy,\quad \mW^* = \argmin_{\mW} \EX[P]{\norm[2]{\vx - \mW \vy}^2} + \lambda \norm[F]{\mW}^2.
\end{align}
Define the risk of an estimate $\widehat{\vx}$ on distribution $P$ with respect to the normalized squared error as $\risk_P(\widehat{\vx}) = \EX[P]{\norm[2]{(\vx - \widehat{\vx})/\sqrt{d_P}}^2}$ and likewise for distribution $Q$.  %
We show that the relationship between $\risk_P(\widehat{\vx}_\lambda)$ and $\risk_Q(\widehat{\vx}_\lambda)$ is captured by a similarity between subspaces $\mU_P$ and $\mU_Q$ that is determined by the principal angles between them. Let $\vtheta \in \reals^{\min\{d_P, d_Q\}}$ be the principal angles between subspaces spanned by the columns of $\mU_P$ and $\mU_Q$, and define
$%
    a = \norm[2]{cos(\vtheta)}^2 / d_Q.
$%

\begin{figure}[t]
\centering
\begin{tabular}{cc}
\begin{tikzpicture}
\begin{axis} [
    title={SNR $= 1$},
    title style = {font=\scriptsize, yshift=-1.6ex},
    xmin=0.47, xmax=0.92, ymin=0.46, ymax=1.31,
    xtick distance = 0.1,
    ytick distance = 0.1,
    yticklabel style = {/pgf/number format/precision=3,
  /pgf/number format/fixed},
    every tick label/.append style={font=\tiny},
    grid = both,
    minor tick num = 1,
    major grid style = {lightgray!25},
    minor grid style = {lightgray!25},
    width = 0.3\textwidth,
    height = 0.3\textwidth,
    x label style={at={(axis description cs:0.5,0.1)}, anchor=north},
    y label style={at={(axis description cs:0.26,0.5)},  anchor=south},
    xlabel = {In-distribution risk}, 
    ylabel = {Out-of-distribution risk},
    label style = {font=\scriptsize},
]
\addplot [
    mark size=1.0pt, smooth, mark=*, brickred, mark options={solid, brickred, fill=brickred}]
    table [x=risk_p_snr1, y=risk_q_a0.0_snr1] {./figures/denoising/average_l2_loss.csv};
\addplot [
    mark size=1.0pt, smooth, mark=*, darkseagreen, mark options={solid, darkseagreen, fill=darkseagreen}]
    table [x=risk_p_snr1, y=risk_q_a0.5_snr1] {./figures/denoising/average_l2_loss.csv};
\addplot [
    mark size=1.0pt, smooth, mark=*, steelblue, mark options={solid, steelblue, fill=steelblue}]
    table [x=risk_p_snr1, y=risk_q_a1.0_snr1] {./figures/denoising/average_l2_loss.csv};
\end{axis}
\end{tikzpicture}
&
\begin{tikzpicture}
\begin{axis} [
    title={SNR $= 100$},
    title style = {font=\scriptsize, yshift=-1.6ex},
    xmin=-0.05, xmax=0.94, ymin=-0.05, ymax=1.11,
    xtick distance = 0.2,
    ytick distance = 0.2,
    yticklabel style = {/pgf/number format/precision=3,
  /pgf/number format/fixed},
    every tick label/.append style={font=\tiny},
    grid = both,
    minor tick num = 1,
    major grid style = {lightgray!25},
    minor grid style = {lightgray!25},
    width = 0.3\textwidth,
    height = 0.3\textwidth,
    x label style={at={(axis description cs:0.5,0.1)}, anchor=north},
    y label style={at={(axis description cs:0.26,0.5)},  anchor=south},
    xlabel = {In-distribution risk}, 
    ylabel = {Out-of-distribution risk},
    label style = {font=\scriptsize},
    legend entries = {$a=0.0$, $a=0.5$, $a=1.0$},
    legend style = {cells={anchor=east}, legend pos=outer north east, font=\scriptsize},
]
\addplot [
    mark size=1.0pt, smooth, mark=*, brickred, mark options={solid, brickred, fill=brickred}]
    table [x=risk_p_snr100, y=risk_q_a0.0_snr100] {./figures/denoising/average_l2_loss.csv};
\addplot [
    mark size=1.0pt, smooth, mark=*, darkseagreen, mark options={solid, darkseagreen, fill=darkseagreen}]
    table [x=risk_p_snr100, y=risk_q_a0.5_snr100] {./figures/denoising/average_l2_loss.csv};
\addplot [
    mark size=1.0pt, smooth, mark=*, steelblue, mark options={solid, steelblue, fill=steelblue}]
    table [x=risk_p_snr100, y=risk_q_a1.0_snr100] {./figures/denoising/average_l2_loss.csv};
\end{axis}
\end{tikzpicture} 
\end{tabular}
\caption{Non-linear relationship between risks $\risk_P(\widehat{\vx})$ and $\risk_Q(\widehat{\vx})$ in the low SNR regime (\textbf{left}) and approximate linear relationship in the high SNR regime (\textbf{right}) in signal denoising. Each curve plots the risks of estimate $\widehat{\vx}_\lambda$ as the parameter $\lambda$ varies. The signal-to-noise ratio is defined as $\text{SNR} = 1/\sigma_P^2$ and we set $\sigma_Q^2 = \sigma_P^2$. Shifts in the subspace is captured by $a = \norm[2]{cos(\vtheta)}^2 / d_Q$.}
\label{fig:denoising_linear_shift}
\end{figure}
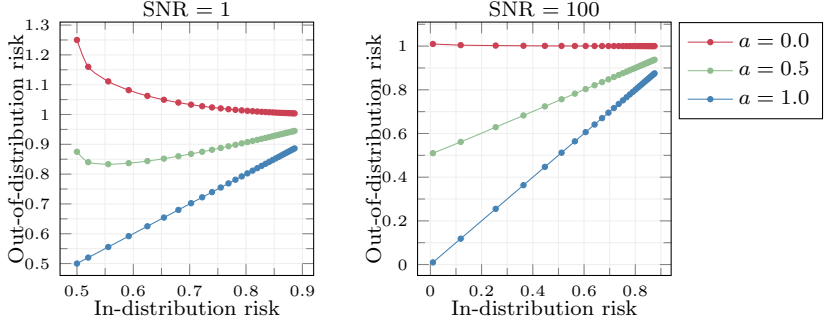

\paragraph{Denoising.}

We start with denoising where the measurement matrix is the identity, i.e., $\mA = \mI$. 
It can be shown that 
$%
    \widehat{\vx}_\lambda(\vy) = \alpha \mU_P \mU_P^\T \vy,
$%
where $\alpha = 1/(1 + \sigma_P^2 + \lambda)$. The following relationship between the risks 
$\risk_P(\widehat{\vx}_\lambda)$ and $\risk_Q(\widehat{\vx}_\lambda)$ holds.

\begin{theorem}
The risks $\risk_P(\widehat{\vx}_\lambda)$ and $\risk_Q(\widehat{\vx}_\lambda)$ of the signal estimate $\widehat{\vx}_\lambda$ obey
\begin{align}
    \risk_Q(\widehat{\vx}_\lambda) = a \risk_P(\widehat{\vx}_\lambda) + (1 - a) + \alpha^2 \left( \frac{d_P}{d_Q} \sigma_Q^2 - a \sigma_P^2 \right),
\end{align}
where $\alpha = 1/(1 + \sigma_P^2 + \lambda)$.
\label{thm:l2_risk_relationship_in_denoising}
\end{theorem}
In general, the relationship between the risks $\risk_P(\widehat{\vx}_\lambda)$ and $\risk_Q(\widehat{\vx}_\lambda)$ is non-linear: it can be shown that $\risk_P(\widehat{\vx}_\lambda) = (1 - \alpha)^2 + \alpha^2 \sigma_P^2$ (see the proof of Theorem~\ref{thm:l2_risk_relationship_in_denoising}), hence the term $\alpha^2 \left( (d_P/d_Q) \sigma_Q^2 - a \sigma_P^2 \right)$ is not a linear function of the risk $\risk_P(\widehat{\vx}_\lambda)$. However, if the noise variances $\sigma_P^2, \sigma_Q^2 \ll 1$, then an approximate linear relation $\risk_Q(\widehat{\vx}_\lambda) \approx a \risk_P(\widehat{\vx}_\lambda) + (1 - a)$ holds.

We illustrate Theorem~\ref{thm:l2_risk_relationship_in_denoising} through a denoising simulation. %
In Figure~\ref{fig:denoising_linear_shift}, we plot the trajectory $(\risk_P(\widehat{\vx}_\lambda), \risk_Q(\widehat{\vx}_\lambda))$, as the parameter $\lambda$ of the estimate $\widehat{\vx}_\lambda$ varies. %
For high SNR the relationship between $\risk_P(\widehat{\vx}_\lambda)$ and $\risk_Q(\widehat{\vx}_\lambda)$ is approximately linear, and for low SNR it is highly nonlinear. 

\paragraph{Compressed sensing.}
We continue with compressed sensing, where the matrix $\mA$ is a random matrix that down-samples the signal $\vx$.
Now the estimate $\widehat{\vx}_\lambda(\vy)$ is only approximately $\alpha \mU_P \mU_P^\T \vy$ due to the random measurement process. However, a similar relationship still holds between the risks.

\begin{theorem}
Let $\mA \in \reals^{n \times d}$ be a random Gaussian matrix with independent entries drawn from the distribution $\mathcal{N}(0, 1/n)$. There exists a constant $c > 0$ such that, for any $0 < \epsilon < 1/d_P$, with probability at least $1 - 4(d_P(d_P + d_Q))\exp(-n\epsilon^2/8)$, 
it holds that 
\begin{align}
    \left| \risk_Q(\widehat{\vx}_\lambda) - a \risk_P(\widehat{\vx}_\lambda) - (1 - a) - \alpha^2 \left( \frac{d_P}{d_Q} \sigma_Q^2 - a \sigma_P^2 \right)  
    \right| \le c \epsilon,
\end{align}
where $\alpha = 1/(1 + \sigma_P^2 + \lambda)$.
\label{thm:l2_risk_relationship_in_compressed_sensing}
\end{theorem}
In the high SNR regime, if the number of measurements $n$ is large enough, then with high probability there is an approximate linear relationship between the risks $\risk_P(\widehat{\vx}_\lambda)$ and $\risk_Q(\widehat{\vx}_\lambda)$.

%% file: discussion.tex
\section{Conclusion}

In this paper, we studied the performance of estimators based on regularized empirical risk minimization trained on a distribution $P$, quantifying how they perform under distribution shifts on a distribution $Q$ for regression, classification, and signal estimation problems. 
We identified conditions under which monotonic relations between the in-distribution risk $\risk_P$ and out-of-distribution risk $\risk_Q$ arise that hold for broad classes of regularized estimators, similarly to the linear risk relationships observed in practice. 

Our findings in this work suggest that the linear and monotonic relations under distribution shifts observed in practice are emergent phenomena that arise from concentration of measure effects in large systems, which reduce the dependence of the risks down to only a single parameter. By identifying necessary and sufficient conditions for monotonic risk relations to exist, and characterizing the form of the monotonic relations, our work enables the principled discussion and investigation of such risk relations in future work.

%% file: appendix-proofs.tex
\makeatletter
\renewcommand{\@seccntformat}[1]{Appendix \csname the#1\endcsname\quad}
\makeatother

\section{Details on the Experimental Results}

Here, we provide further details on the numerical experiments in the main body.

\subsection{Experimental Details for Object Detection}
\label{sec:experiment_details_for_object_detection}

In this section, we describe the details of the object detection experiment from 
Section~\ref{sec:linear_shift_in_regression_and_motivation_for_subspace_model}.

The models we evaluate are from \texttt{torchvision.models} and public github repositories:
\begin{itemize}
    \item RetinaNet~\citep{lin17}: RetinaNet ResNet-50 FPN
    \item Mask R-CNN~\citep{he17}: Mask R-CNN ResNet-50 FPN
    \item SSD~\citep{liu16}: SSD300 VGG16, SSDlite320 MobileNetV3-Large
    \item Faster R-CNN~\citep{ren15}: Faster R-CNN ResNet-50 FPN, Faster R-CNN MobileNetV3-Large FPN, Faster R-CNN MobileNetV3-Large 320 FPN
    \item Keypoint R-CNN~\citep{he17}: Keypoint R-CNN ResNet-50 FPN
    \item YOLOv5 ~\citep{redmon16, jocher20}: YOLOv5n, YOLOv5s, YOLOv5m, YOLOv5l, YOLOv5x
\end{itemize}
These model are trained on the COCO 2017 training set~\citep{lin14}. We take the trained models and evaluate their performances on the COCO 2017 validation set and the VOC 2012 training/validation set ~\citep{everingham10}. Instead of using the standard metric for object detection---the mean average precision (mAP), which is the area under the precision-recall curve averaged over all classes---we consider the mean squared error in bounding box coordinates and only the \texttt{person} class. The predicted and the ground truth bounding box coordinates are normalized by the height and width of individual image. All models are evaluated using an NVIDIA A40 GPU.

To analyze the spectrum of the feature space of YOLOv5, we collect feature vectors through the following procedure. For each image in each evaluation set, we record the ground truth \texttt{person} objects that are correctly detected by \emph{all} models listed above with an IOU threshold greater than or equal to $0.2$. Then for each commonly detected ground truth object, we consider the prediction that has the largest IOU with the ground truth bounding box as the true positive. We then extract the feature vectors corresponding to these true positive predictions from the $24$\textsuperscript{th} layer of YOLOv5. This procedure is illustrated in Figure~\ref{fig:yolov5_visualization}.

\begin{figure}[hb!]
\centering
\setlength\tabcolsep{1.5pt} %
\includegraphics[width=0.96\textwidth]{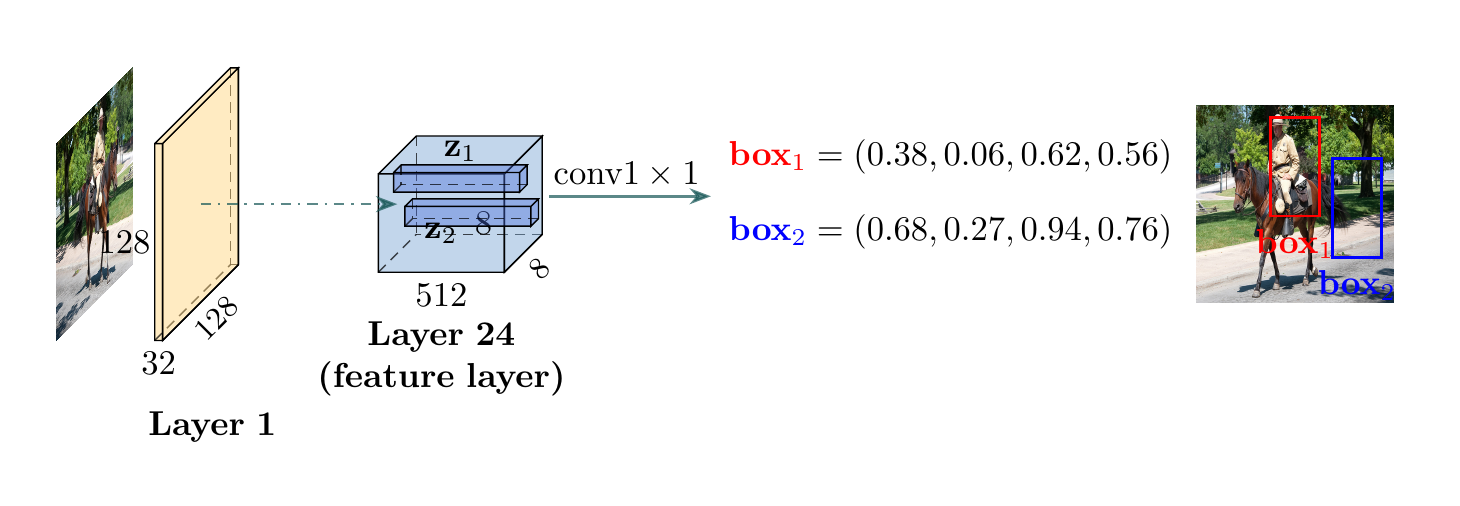}
\caption{Visualization of feature extraction from YOLOv5: only feature vectors that correspond to true positive predictions are recorded for feature space analysis. The prediction of $\textcolor{red}{\text{box}_1}$, which is based on the feature vector $\vz_1$, is true positive, while the prediction of $\textcolor{blue}{\text{box}_2}$, which is based on the feature vector $\vz_2$, is false positive. We record the feature vector $\vz_1$ and discard the feature vector $\vz_2$. Similarly, predictions in other grid positions in the $8 \times 8$ grid of the feature layer are not recorded if they do not correspond to a true positive prediction, since these feature vectors do not contain much information about the correct bounding box coordinates.}
\vspace{10pt}
\label{fig:yolov5_visualization}
\end{figure}

At the end, relevant feature vectors across the same evaluation set are stacked together and we obtain two sets of feature vectors
$\mathcal{Z}_{\text{in}} = \{\vz_j^{(i)} \in \reals^{512}: i \in [N_{\text{in}}], j \in [K_{\text{in}}^{(i)}]\}$ and $\mathcal{Z}_{\text{out}} = \{\vz_j^{(i)} \in \reals^{512}: i \in [N_{\text{out}}], j \in [K_{\text{out}}^{(i)}]\}$ for the COCO 2017 and VOC 2012 evaluation dataset respectively, where $\vz_j^{(i)}$ is the $j$\textsuperscript{th} true positive prediction on image $i$, $N_{\text{in}}$ and $N_{\text{out}}$ are the numbers of images of the respective dataset and $K_{\text{in}}^{(i)}$ and $K_{\text{out}}^{(i)}$ are the number of true positive predictions on the $i$\textsuperscript{th} image respectively.

We make a few comments:
\begin{itemize}
    \item[1] We only consider \emph{common} \emph{true positive} predictions: (1) for false positive and true negative predictions, there is no object to predict, hence the feature vectors contain no information for the regression task; (2) for false negative predictions, either the squared errors of the predicted coordinates are large since the IOU is lower than the $0.2$ threshold, or they have lower confidence than another prediction which is true positive, so we simply exclude the corresponding feature vectors as they do not provide much useful information; (3) only common true positive predictions are considered so that all models make predictions on the same set of feature vectors.
    \item[2] YOLOv5 uses multiple layers (the $18$\textsuperscript{th} and $21$\textsuperscript{st} layers in addition to the $24$\textsuperscript{th} layer) as input to the bounding box prediction layer, but we find that most common true positive predictions are based on the $24$\textsuperscript{th} layer, probably due to the fact that this layer has a spacial dimension $8 \times 8$, where most ground truth objects size fit into, while the other layers have special dimension $16 \times 16$ and $32 \times 32$ matching small and tiny objects, which are relatively harder to predict.
\end{itemize}

\subsection{Experimental Details for Digit Classification}
\label{sec:experiment_details_for_digit_classification}

In this section, we describe the details of the even vs odd handwritten digit classification experiment in Figure~\ref{fig:classification-risk-shifts} (right).

We consider a binary classification task of classifying even versus odd digits on the MNIST~\citep{lecun10} dataset and ARDIS~\citep{kusetogullari20} dataset IV. The ARDIS dataset is a new image-based handwritten historical digit dataset extracted from Swedish church records, which induces a natural distribution shift from the widely-used MNIST dataset. The ARDIS dataset IV has the same image size as the MNIST dataset with white digits in black background. The following figure shows examples of digits from both datasets.

\begin{figure}[!h]
    \centering
    \begin{tabular}{ccc}
        \includegraphics[width=0.85\textwidth]{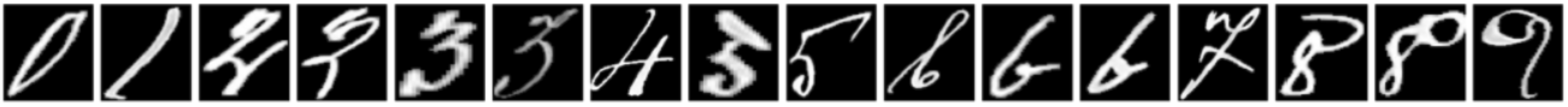} \\
        (a) \\
        \includegraphics[width=0.85\textwidth]{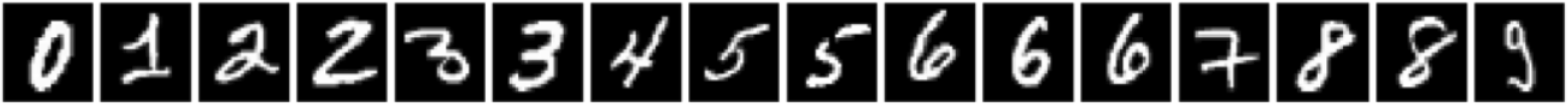} \\
        (b)
    \end{tabular}
    \caption{Examples of digits: (a) ARDIS and (b) MNIST.}
    \label{fig:examples_of_ardis_and_mnist_digits}
\end{figure}

The models we evaluate are from \texttt{torchvision.models}:
\begin{itemize}
    \item AlexNet~\citep{krizhevsky12}
    \item VGG~\citep{simonyan15}: VGG11, VGG16
    \item ResNet~\citep{he16}: ResNet18, ResNet50
    \item DenseNet~\citep{huang17}: DenseNet121, DenseNet161
\end{itemize}

Since the models we evaluate are originally designed for ImageNet~\citep{deng09} classification where the image sizes are larger, we resize the MNIST and ARDIS digits from $28 \times 28$ to $75 \times 75$. We train the model listed above on MNIST training set using the Adam optimizer with an initial learning rate $10^{-4}$ and a batch size $10$ and a learning rate scheduler with a step size $10$ epochs and a learning rate decay factor $0.1$. The models at the top right corner of Figure~\ref{fig:classification-risk-shifts}(right) are trained for $20$ epochs. Intermediate models are obtained by early stopping when validation accuracy first reaches $0.5, 0.6, 0.7, 0.8$ and $0.9$. Each model is trained eight times with random initialization and with random shuffling of the training data, using different random seeds. All models are trained on an NVIDIA A40 GPU.

\section{Intuition Regarding Feature Scaling}
\label{sec:motivation-feature-scaling}

In both regression and classification, we find that the risk relations depend on the scaling of the features. We give intuition regarding where these can be seen in practice for two settings of real data that we consider in this paper.

\paragraph{Task-independent feature scaling.} The scaling of the features may be uncorrelated with the ground truth coefficients $\vbeta^*$, such as in the subspace shift case $\mSigma = \rho \mPi_Q$. An example in real data is that the principal components of the learned feature space of YOLOv5 model on VOC 2012 have uniformly larger magnitudes than those on COCO 2017, as can be seen from Figure~\ref{fig:bounding_box_prediction}.

\paragraph{Task-dependent feature scaling.} The scaling of the features may be correlated with the ground truth coefficients $\vbeta^*$. A motivating example in real data is the MNIST and ARDIS handwritten digit datasets. The universal ground truth labeling function for both of these datasets is the same (humans can classify digits from both datasets very well) and conceivably relies on a complex combination of features involving stroke and loop placement. Such features are for example the types of features found by nonlinear embedding techniques such as Isomap~\cite{tenenbaum00}. While both strokes and loops are present in both datasets, we observe that some of these occur with more frequency and intensity in one dataset versus the other. For example, italics and embellishments are much more common in ARDIS than in MNIST, as can be seen from Figure~\ref{fig:examples_of_ardis_and_mnist_digits}. We imagine that in feature space, this corresponds to a larger scaling of these features. 

\section{Proof of Theorem~\ref{thm:sufficient_conditions_for_linear_shift_in_regression}}

\label{sec:proof_of_conditions_for_linear_shift_in_regression}

In this section, we prove the main results on linear relations in linear regression in finite dimensions.

The proof of Theorem~\ref{thm:sufficient_conditions_for_linear_shift_in_regression} is based on the following sufficient and necessary condition for a linear relationship between risks $\risk_P(\widehat{\vbeta}_\lambda)$ and $\risk_Q(\widehat{\vbeta}_\lambda)$.
\begin{theorem}[Sufficient and necessary condition]
The risk 
$\risk_Q(\widehat{\vbeta}_\lambda)$ of estimator $\widehat{\vbeta}_\lambda$ is an affine function of $\risk_P(\widehat{\vbeta}_\lambda)$, i.e., there exist $a$ and $b$ such that $\risk_Q(\widehat{\vbeta}_\lambda) = a \risk_P(\widehat{\vbeta}_\lambda) + b$, if and only if 
\begin{align}
    \vbeta^{*\T} \mPi_P \mSigma_Q \mPi_{P^\perp} \vbeta^* = 0,
\end{align}
where $\mPi_P = \mU_P \mU_P^\T$ is a orthonormal projection onto the subspace spanned by the orthonormal matrix $\mU_P$, $\mPi_{P^\perp} = \mI - \mPi_P$ is a projection onto the orthogonal complement, and $\mSigma_Q = \EX[Q]{\vx \vx^\T}$.
\label{thm:sufficient_necessary_conditions_for_linear_shift_in_regression}
\end{theorem}

If condition $(a)$ holds, then $\mSigma_Q$ and $\mPi_P$ commute, and therefore $\mPi_P \mSigma_Q \mPi_{P^\perp} = \mSigma_Q \mPi_P \mPi_{P^\perp} = \vzero$. If condition $(b)$ holds, then $\mPi_{P^\perp} \vbeta^{*\T} = \vzero$. In both cases, the term marked with $(*)$ in the proof of Theorem~\ref{thm:sufficient_necessary_conditions_for_linear_shift_in_regression} becomes zero and the linear relationship $\risk_Q(\vbeta) = a \risk_P(\vbeta) + b$ holds with slope and intercept 
\begin{align}
    a &= \frac{\vbeta^{*\T} \mPi_P \mSigma_Q \mPi_P \vbeta^*}{\vbeta^{*\T} \mPi_P \vbeta^*}, \\
    b &= \vbeta^{*\T} \left( \mSigma_Q  - a  \mPi_P \right) \vbeta^* + \sigma_Q^2 - a \sigma_P^2.
\end{align}
If condition $(c)$ holds, then expectation of the term marked with $(*)$ is zero:
\begin{align}
    \EX[\vbeta^*]{\vbeta^{*\T} \mPi_P \mSigma_Q \mPi_{P^\perp} \vbeta^*}
    = \tr\left( \mPi_P \mSigma_Q \mPi_{P^\perp} \EX[\vbeta^*]{\vbeta^* \vbeta^{*\T}} \right) 
    = \tr\left( \mSigma_Q \mPi_{P^\perp} \mPi_P \right)
    = 0,
\end{align}
and $\EX[\vbeta^*]{\risk_Q(\widehat{\vbeta}_\lambda)} = a \EX[\vbeta^*]{\risk_P(\widehat{\vbeta}_\lambda)} + b$ with slope and intercept
\begin{align}
    a &= \frac{\tr\left( \mPi_P \mSigma_Q \mPi_P \right)}{ \tr\left( \mPi_P \right)}, \\
    b &= \tr \left( \mSigma_Q  - a  \mPi_P \right) + \sigma_Q^2 - a \sigma_P^2.
\end{align}
This concludes the proof of Theorem~\ref{thm:sufficient_conditions_for_linear_shift_in_regression}. It remains to prove Theorem~\ref{thm:sufficient_necessary_conditions_for_linear_shift_in_regression}.

\paragraph{Proof of Theorem~\ref{thm:sufficient_necessary_conditions_for_linear_shift_in_regression}.}
The idea is to relate the risks $\risk_P(\vbeta)$ and $\risk_Q(\vbeta)$ with the help of the parameter $\alpha$. We start by expressing the risk of $\vbeta = \alpha \mU_P \mU_P^\T \vbeta^*$ on distribution $P$ as a function of $\alpha$
\begin{align}
    \risk_P(\vbeta) 
    &= \EX[P]{(y - \vx^\T \vbeta)^2} \\
    &= (\vbeta^* - \vbeta)^\T \EX[P]{\vx \vx^\T} (\vbeta^* - \vbeta) + \sigma_P^2 \\
    &= \vbeta^{*\T} (\mI - \alpha \mU_P \mU_P^\T) \mU_P \EX[P]{\editedinline{\vc_P \vc_P^\T}} \mU_P^\T (\mI - \alpha \mU_P \mU_P^\T) \vbeta^* + \sigma_P^2 \\
    &= \vbeta^{*\T} \mU_P \EX[P]{\editedinline{\vc_P \vc_P^\T}} \mU_P^\T \vbeta^* + (\alpha^2 - 2\alpha) \vbeta^{*\T} \mU_P \EX[P]{\editedinline{\vc_P \vc_P^\T}} \mU_P^\T \vbeta^* + \sigma_P^2 \\
    &= \vbeta^{*\T} \mPi_P \vbeta^* + (\alpha^2 - 2\alpha) \vbeta^{*\T} \mPi_P \vbeta^* + \sigma_P^2.
\end{align}
Similarly, on distribution $Q$
\begin{align}
    \risk_Q(\vbeta) 
    &= \EX[Q]{(y - \vx^\T \vbeta)^2} \\
    &= (\vbeta^* - \vbeta)^\T \EX[Q]{\vx \vx^\T} ((\vbeta^* - \vbeta)) + \sigma_Q^2 \\
    &= \vbeta^{*\T} (\mI - \alpha \mU_P \mU_P^\T) \mSigma_Q (\mI - \alpha \mU_P \mU_P^\T) \vbeta^* + \sigma_Q^2 \\
    &= \vbeta^{*\T} \mSigma_Q \vbeta^* + \alpha^2 \vbeta^{*\T} \mU_P \mU_P^\T \mSigma_Q \mU_P \mU_P^\T \vbeta^* \\
    & \quad\quad\quad - 2\alpha \vbeta^{*\T} \mU_P \mU_P^\T \mSigma_Q (\mU_P \mU_P^\T + \mU_{P^\perp} \mU_{P^\perp}^\T) \vbeta^* + \sigma_Q^2 \\
    &= \vbeta^{*\T} \mSigma_Q \vbeta^* + (\alpha^2-2\alpha) \vbeta^{*\T} \mU_P \mU_P^\T \mSigma_Q \mU_P \mU_P^\T \vbeta^* \\
    & \quad\quad\quad - 2\alpha \vbeta^{*\T} \mU_P \mU_P^\T \mSigma_Q \mU_{P^\perp} \mU_{P^\perp}^\T \vbeta^* + \sigma_Q^2 \\
    &= \vbeta^{*\T} \mSigma_Q \vbeta^* + (\alpha^2-2\alpha) \vbeta^{*\T} \mPi_P^\T \mSigma_Q \mPi_P \vbeta^* \\
    & \quad\quad\quad - 2\alpha \underbrace{\vbeta^{*\T} \mPi_P \mSigma_Q \mPi_{P^\perp} \vbeta^*}_{(*)} + \sigma_Q^2.
\end{align}
Since the risk $\risk_P(\vbeta)$ depends linearly on $\alpha^2 - 2\alpha$, a linear relationship between $\risk_Q(\vbeta)$ and $\risk_P(\vbeta)$ is equivalent to $\risk_Q(\vbeta)$ being also linearly dependent on $\alpha^2 - 2\alpha$. Hence, it is sufficient and necessary that the term marked with $(*)$ is zero.

\section{Proof of Theorem~\ref{thm:monotonic-risks}}
\label{sec:thm:monotonic-risks:proof}

The proof of Theorem~\ref{thm:monotonic-risks} involves the following steps:

\begin{enumerate}[label=(Step~\arabic*), ref=Step~\arabic*]
    \item \textit{Asymptotics.} We invoke the result of \citet{loureiro21} (Theorem~\ref{thm:gerbelot-asymp-cgmt}) to characterize the covariances of the decision functions of the estimator and ground truth in terms of three parameters (Lemma~\ref{lem:asymp-cgmt-cov}).
    \label{proof-step:asymptotics}
    \item \textit{Monotonicity for linear risks.} We prove a generic result for any risk that is parameterized as an affine function of some of its parameters, providing necessary and sufficient conditions for a monotonic relation (Lemma~\ref{lem:linear-risk-monotonic}).
    \label{proof-step:monotonicity}
    \item \textit{Specific metrics.} We apply the generic result in Lemma~\ref{lem:linear-risk-monotonic} to squared error and misclassification error to obtain the most general results (Theorems~\ref{thm:monotonic-risk-squared-error} and \ref{thm:monotonic-risk-classification-error}).
    \label{proof-step:specific-metrics}
    \item \textit{Simplifying assumptions.} To aid in interpretability, we apply Assumption~\ref{assump:asymp-cgmt:gt-cov} to simplify the necessary and sufficient conditions.
    \label{proof-step:simplifying}
\end{enumerate}

\subsection{\ref{proof-step:asymptotics}: Asymptotics}

Assumption~\ref{assump:asymp-cgmt:loss} states that the loss function is pseudo-Lipschitz continuous of order 2, which is defined as follows.
\begin{definition}[Pseudo-Lipschitz continuity]
    \label{def:pl-continuous}
    For a given $p \geq 1$, a function $\vf \colon \reals^r \to \reals^s$ is called pseudo-Lipschitz of order $p$ if there exists a constant $C > 0$ such that for all $\vx_1, \vx_2 \in \reals^r$,
    \begin{align}
        \norm{\vf(\vx_1) - \vf(\vx_2)} \leq C \norm{\vx_1 - \vx_2} \big(1 + \norm{\vx_1}^{p-1} + \norm{\vx_2}^{p-1} \big).
    \end{align}
\end{definition}
We also need the definition of the proximal operator of a function.
\begin{definition}[Proximal operator]
    The proximal operator of a function $f \colon \reals^r \to \reals$ is the unique minimizer of the following objective:
    \begin{align}
         \prox{f}{\vz} \defeq \argmin_\vx f(\vx) + \tfrac{1}{2} \norm[2]{\vx - \vz}^2.
    \end{align}
\end{definition}
We finally replace Assumption~\ref{assump:asymp-cgmt} with a slightly more general version, that implies Assumption~\ref{assump:asymp-cgmt}.

\begin{assumption}{A*}[General setup]
    \label{assump:asymp-cgmt-general}
    Assumption~\ref{assump:asymp-cgmt} holds with  Assumption~\ref{assump:asymp-cgmt:gt-cov} replaced as follows.
    \begin{enumerate}[label=\normalfont{(A{\arabic*}*)}, ref=A{\arabic*}*]
        \setcounter{enumi}{2}
        \item The ground truth coefficients $\vbeta^*$ are deterministic, or they are random with sub-Gaussian one-dimensional marginals independent of $\setD$, and the spectral distribution of $\mSigma_P$ converges with bounded eigenvalues, such that $\tfrac{1}{d}\vbeta^{*\transp} \mSigma_P \vbeta^*$ and $\tfrac{1}{d} \norm[2]{\vbeta^*}^2$ converge to finite nonzero limits as $d \to \infty$.
        \label{assump:asymp-cgmt:gt-cov-general}
    \end{enumerate}
\end{assumption}
Armed with Assumption~\ref{assump:asymp-cgmt-general}, we are now ready to re-state Theorem 5 of \citet{loureiro21} in our notation.
\begin{theorem}
    \label{thm:gerbelot-asymp-cgmt}
    Under Assumption~\ref{assump:asymp-cgmt-general}, there exist scalar coefficients $a \in \reals$, $b, c, C_1, C_2, C_3 > 0$ such that for any pseudo-Lipschitz function $h \colon \reals^d \to \reals$ of order 2 and any $0 < \epsilon < C_1$, with probability at least $1 - \tfrac{C_2}{\epsilon^2} e^{-C_3 n \epsilon^4}$, the estimator $\widehat{\vbeta}$ in \eqref{eq:erm-ridge} satisfies
    \begin{align}
        \Bigg| h \paren{\tfrac{1}{\sqrt{d}} \widehat{\vbeta}} - h \paren{ \tfrac{1}{\sqrt{d}} \mSigma_P^{-1/2} \prox{\tfrac{1}{b} \tfrac{1}{2}\norm[2]{\mSigma_P^{-1/2} \cdot}^2 }{\tfrac{a}{b} \mSigma_P^{1/2} \vbeta^* + \tfrac{\sqrt{c}}{b} \vg} } \Bigg| < \epsilon,
    \end{align}
    where $\vg \sim \normal(0, \mI_d)$ is independent of $\vbeta^*$.
\end{theorem}
Combining this theorem with
\begin{align}
\prox{\tfrac{1}{2b} \norm[2]{\mSigma_P^{-1/2} \cdot}^2 }{\vz} = \inv{\mI_d + \tfrac{1}{b} \mSigma_P^{-1}} \vz
\end{align}
and using the Borel-Cantelli lemma, extending from a single function $h$ to a sequence of functions that are uniformly pseudo-Lipschitz of order 2, we obtain the following corollary.
\begin{corollary}
\label{cor:asymp-cgmt}
Under Assumption~\ref{assump:asymp-cgmt-general}, there exist $a \in \reals$, $b, c > 0$ such that for any pseudo-Lipschitz functions $h_d \colon \reals^d \to \reals$ of order 2 with uniform constant $C > 0$, the following holds almost surely for the estimator $\widehat{\vbeta}$ in \eqref{eq:erm-ridge}:
\begin{align}
    \lim_{d \to \infty} h_d \paren{\tfrac{1}{\sqrt{d}} \widehat{\vbeta}} = \lim_{d \to \infty} h_d \paren{\tfrac{1}{\sqrt{d}} \mSigma_P^{1/2} \inv{\mSigma_P + b \mI_d} (a \mSigma_P^{1/2} \vbeta^* + \sqrt{c} \vg)},
\end{align}
where $\vg \sim \normal(0, \mI_d)$ is independent of $\vbeta^*$.
\end{corollary}
Finally, we obtain the form of the limiting covariances that we need for our proof.
\begin{lemma}
    \label{lem:asymp-cgmt-cov}
    Under Assumption~\ref{assump:asymp-cgmt-general}, as $n, d \to \infty$, and assuming the limits below exist for any $a \in \reals$, $b, c > 0$, there exist $a \in \reals$, $b, c > 0$ such that the estimator in \eqref{eq:erm-ridge} has decision functions converging almost surely to $\widehat{Z}_P$ and $\widehat{Z}_Q$ that satisfy
    \begin{gather}
        \expect{Z_P^{*2}} = \lim_{d \to \infty} \tfrac{1}{d} \vbeta^{*\transp} \mSigma_P \vbeta^*,\\
        \expect{Z_P^* \widehat{Z}_P} = \lim_{d \to \infty} \tfrac{a}{d} \vbeta^{*\transp} \mSigma_P^2 \inv{\mSigma_P + b \mI_d} \vbeta^*, \\
        \expect{\widehat{Z}_P^2} = \lim_{d \to \infty} \tfrac{a^2}{d} \vbeta^{*\transp} \mSigma_P^3 \inv[2]{\mSigma_P + b \mI_d} \vbeta^* + \tfrac{c}{d} \tr \bracket{\mSigma_P^2 \inv[2]{\mSigma_P + b \mI_d}}
        \\
        \expect{Z_Q^{*2}} = \lim_{d \to \infty} \tfrac{1}{d} \vbeta^{*\transp} \mSigma_Q \vbeta^*,\\
        \expect{Z_Q^* \widehat{Z}_Q} = \lim_{d \to \infty} \tfrac{a}{d} \vbeta^{*\transp} \mSigma_Q \mSigma_P \inv{\mSigma_P + b \mI_d} \vbeta^*, \\
        \expect{\widehat{Z}_Q^2} = \lim_{d \to \infty} \tfrac{a^2}{d} \vbeta^{*\transp} \mSigma_P \inv{\mSigma_P + b \mI_d} \mSigma_Q \mSigma_P \inv{\mSigma_P + b \mI_d} \vbeta^* + \tfrac{c}{d} \tr \bracket{\mSigma_Q \mSigma_P \inv[2]{\mSigma_P + b \mI_d}}.
    \end{gather}
\end{lemma}
\begin{proof}
The variances $\expect{Z_P^{*2}}$ and $ \expect{Z_Q^{*2}}$ are simply defined as stated. 
For $\expect{Z_P^* \widehat{Z}_P}$, observe that the decision functions $\vx^\transp \vbeta^*$ and $\vx^\transp \widehat{\vbeta}$ have correlation
\begin{align}
    \expect[\vx \sim P]{(\vx^\transp \vbeta^*) (\vx^\transp \widehat{\vbeta})} = \tfrac{1}{d} \vbeta^{*\transp} \mSigma_P \widehat{\vbeta}.
\end{align}
The functions in the sequence $h_d(\vu) = \tfrac{1}{\sqrt{d}} \vbeta^{*\transp} \mSigma_P \vu$ are uniformly Lipschitz since $\tfrac{1}{d} \vbeta^{*\transp} \mSigma_P \vbeta^*$ converges and $\mSigma_P$ has uniformly bounded eigenvalues almost surely, so we can apply Corollary~\ref{cor:asymp-cgmt} to obtain the stated result. Similarly, for $\expect{\widehat{Z}_P^2}$, the functions $h_d(\vu) = \vu^\transp \mSigma_P \vu$ are pseudo-Lipschitz continuous of order 2 with uniform constant $C$. 
The calculation is analogous for $\expect{Z_Q^* \widehat{Z}_Q}$ and $\expect{\widehat{Z}_Q^2}$, applying the functions
$h_d(\vu) = \tfrac{1}{\sqrt{d}} \vbeta^{*\transp} \mSigma_Q \vu$ and
$h_d(\vu) = \vu^{\transp} \mSigma_Q \vu$, respectively.
\end{proof}

\subsection{\ref{proof-step:monotonicity}: Monotonicity for Linear Risks}

Proving necessary and sufficient conditions for arbitrary risks is not a trivial task. However, if the risk has a \emph{linear} structure in some of the free parameters (perhaps after some invertible transformation), we can exploit this linearity to show that any risk relation must be affine. 
\begin{lemma}
    \label{lem:linear-risk-monotonic}
    Consider the following functions defined on $\setA \times \setB$ for open sets $\setA \subseteq \reals^{k_A}$ and $\setB \subseteq \reals^{k_B}$:
    \begin{align}
        R_P(\va, \vb) = h(\vw(\va)^\transp \vv_P(\vb) + v_P^0(\vb))
        \quad \text{and} \quad
        R_Q(\va, \vb) = h(\vw(\va)^\transp \vv_Q(\vb) + v_Q^0(\vb)),
    \end{align} where
    \begin{itemize}
        \item $h \colon \reals \to \reals$ is a monotonically increasing or decreasing function,
        \item $(\vw(\va), 1) \in \reals^{k_W + 1}$ is a vector of linearly independent scalar functions of $\va$ over $\setA$,
        \item $\vv_P$, $\vv_Q$, $v_P^0$, and $v_Q^0$ are differentiable functions of $\vb$, and $\vv_P(\vb) \neq \vzero$ for all $\vb \in \setB$.
    \end{itemize}
The following statements are equivalent: 
\begin{enumerate}[label=(\roman*),ref=\emph{(\roman*)}]
\item There exists a monotonically increasing
function $u \colon \reals \to \reals$ such that $R_Q(\va, \vb) = u(R_P(\va, \vb))$ for all $\va, \vb \in \setA \times \setB$. 
\label{cond:monotonic-exists}
\item
There exists $\rho > 0$, $u_0 \in \reals$ such that for all $\vb \in \setB$, $\vv_Q(\vb) = \rho \vv_P(\vb)$ and $v_Q^0(\vb) = \rho v_P^0(\vb) + u_0$.
\label{cond:scaling-exists}
\end{enumerate}
Furthermore, if $u$ exists, it has the form $u(t) = h \paren{\rho h^{-1}(t) + u_0}$.
\end{lemma}
\begin{proof}
We first show that condition \ref{cond:scaling-exists} implies \ref{cond:monotonic-exists}. 
Denote $t(\va,\vb) = \vw(\va)^\transp \vv_P(\vb) + v_P^0(\vb)$ and note that 
 condition \ref{cond:scaling-exists} implies that $R_P(\va,\vb) = h(t(\va,\vb))$ and $R_Q(\va,\vb) = h(\rho t(\va,\vb) + u_0 )$. 
Next, note that the function $\tilde{u}(t) = \rho t + u_0$, $\rho>0$ is monotonically increasing, and so is $u = h \circ \tilde{u} \circ h^{-1}$, since a composition of increasing and decreasing functions is increasing if the number of decreasing functions is even, and $h$ and $h^{-1}$ are either both increasing or both decreasing.
Thus, condition \ref{cond:scaling-exists} implies \ref{cond:monotonic-exists}. 

It remains to show that condition \ref{cond:monotonic-exists} implies \ref{cond:scaling-exists}. 
For this, we identify necessary conditions for \ref{cond:monotonic-exists} to hold. First note that by a similar argument to the \ref{cond:scaling-exists}$\implies$\ref{cond:monotonic-exists} case , \ref{cond:monotonic-exists} holds if and only if there is a monotonic $\tilde{u}$ such that
\begin{align}
    \label{eq:lem:monotonic-existence:u-tilde}
    \vw(\va)^\transp \vv_Q(\vb) + v_Q^0(\vb) = \tilde{u}(\vw(\va)^\transp \vv_P(\vb) + v_P^0(\vb)).
\end{align}
In the following, we show that for this equation to hold, the function $\tilde u$ must have the form 
$\tilde{u}(t) = \rho t + u_0$ for $\rho>0$.

We begin by taking the gradient of both sides of equation~\eqref{eq:lem:monotonic-existence:u-tilde} with respect to $\vw(\va)$, giving the condition
\begin{align}
    \label{eq:lem:monotonic-existence:v-scaled}
    \vv_Q(\vb) = \tilde{u}'(\vw(\va)^\transp \vv_P(\vb) + v_P^0(\vb)) \vv_P(\vb).
\end{align}
Since the above equation must hold for all $\va, \vb \in \setA \times \setB$, the derivative $\tilde{u}'\colon \reals \to \reals$ must be a function of $\vb$ only---let us write this as $\rho(\vb) \defeq \tilde{u}'(\vw(\va)^\transp \vv_P(\vb) + v_P^0(\vb))$. We can additionally take the gradients of equation~\eqref{eq:lem:monotonic-existence:u-tilde} with respect to $\vb$:
\begin{align}
    \nabla_\vb \vv_Q(\vb) \vw(\va) + \nabla_\vb v_Q^0(\vb) = \tilde{u}'(\vw(\va)^\transp \vv_P(\vb) + v_P^0(\vb)) \paren{\nabla_\vb \vv_P(\vb) \vw(\va) + \nabla_\vb v_P^0(\vb)}.
\end{align}
We can rewrite this equation as
\begin{align}
    \bracket{\nabla_\vb (\vv_Q(\vb), v_Q^0(\vb)) - \rho(\vb) \nabla_\vb (\vv_P(\vb), v_P^0(\vb))} (\vw(\va), 1) = \vzero.
\end{align}
In this form, we can see that because $(w(\va), 1)$ is a vector of linearly independent functions over $\va \in \setA$, the only solutions to this equation are the trivial solutions which satisfy 
\begin{align}
    \label{eq:lem:monotonic-existence:nabla-scaled}
    \nabla_\vb (\vv_Q(\vb), v_Q^0(\vb)) = \rho(\vb) \nabla_\vb (\vv_P(\vb), v_P^0(\vb)).
\end{align}
Returning to equation~\eqref{eq:lem:monotonic-existence:v-scaled}, we can now take its gradient with respect to $\vb$, yielding
\begin{align}
    \nabla_\vb \vv_Q(\vb) = \paren{\nabla_\vb \rho(\vb)} \vv_P(\vb)^\transp + \rho(\vb) \nabla_\vb \vv_P(\vb),
\end{align}
which, combined with equation~\eqref{eq:lem:monotonic-existence:nabla-scaled} implies that $\paren{ \nabla_\vb \rho(\vb)} \vv_P(\vb)^\transp = \vzero$, implying that $\nabla_\vb \rho(\vb) = \vzero$ since $\vv_P(\vb) \neq \vzero$ by assumption. Thus, $\rho(\vb)$ is constant as a function of $\vb$, implying that $\tilde{u}$ is an affine function; let us therefore write $\tilde{u}(t) = \rho t + u_0$. Then we can rewrite equation~\eqref{eq:lem:monotonic-existence:u-tilde} as
\begin{align}
    \bracket{(\vv_Q(\vb), v_Q^0(\vb)) - \rho (\vv_P(\vb), v_P^0(\vb)) - (\vzero, u_0)}^\transp (\vw(\va), 1) = 0.
\end{align}
By linear independence again, this equation can have only the trivial solution, implying that $v_Q^0(\vb) = \rho v_P^0(\vb) + u_0$. Lastly, this mapping is monotonically increasing only if $\rho > 0$.
\end{proof}

\subsection{\ref{proof-step:specific-metrics}: Squared Error}

We start with the simpler case of squared error. We first introduce notation to simplify expressions. Let
\begin{align}
    \expect{Z_P^{*2}} &= \Omega_P, \quad
    \expect{Z_P^* \widehat{Z}_P} = a \Gamma_P(b), \quad
    \expect{\widehat{Z}_P^2} = a^2 \Lambda_P(b) + c \Theta_P(b), \\
    \expect{Z_Q^{*2}} &= \Omega_Q, \quad
    \expect{Z_Q^* \widehat{Z}_Q} = a \Gamma_Q(b), \quad
    \expect{\widehat{Z}_Q^2} = a^2 \Lambda_Q(b) + c \Theta_Q(b),
\end{align}
where
\begin{equation}
\begin{gathered}
    \Omega_P \defeq \lim_{d \to \infty}     \tfrac{1}{d} \vbeta^{*\transp} \mSigma_P \vbeta^*, \quad 
    \Gamma_P(b) \defeq \lim_{d \to \infty} \tfrac{1}{d} \vbeta^{*\transp} \mSigma_P^2 \inv{\mSigma_P + b \mI_d} \vbeta^*, \\
    \Lambda_P(b) \defeq \lim_{d \to \infty} \tfrac{1}{d} \vbeta^{*\transp} \mSigma_P^3 \inv[2]{\mSigma_P + b \mI_d} \vbeta^*, \quad
    \Theta_P(b) \defeq \lim_{d \to \infty} \tfrac{1}{d} \tr \bracket{\mSigma_P^2 \inv[2]{\mSigma_P + b \mI_d}}, \\
    \Omega_Q \defeq \lim_{d \to \infty}     \tfrac{1}{d} \vbeta^{*\transp} \mSigma_Q \vbeta^*, \quad
    \Gamma_Q(b) \defeq \lim_{d \to \infty}
    \tfrac{1}{d} \vbeta^{*\transp} \mSigma_Q \mSigma_P \inv{\mSigma_P + b \mI_d} \vbeta^*, \\
    \Lambda_Q(b) \defeq \lim_{d \to \infty} \tfrac{1}{d} \vbeta^{*\transp} \mSigma_P \inv{\mSigma_P + b \mI_d} \mSigma_Q \mSigma_P \inv{\mSigma_P + b \mI_d} \vbeta^*, \\
    \Theta_Q(b) \defeq \lim_{d \to \infty} 
    \tfrac{1}{d} \tr \bracket{\mSigma_Q \mSigma_P \inv[2]{\mSigma_P + b \mI_d}}.
\end{gathered}
\label{eq:brief-notation}
\end{equation}
We now prove the squared error case in the following theorem.
\begin{theorem}
    \label{thm:monotonic-risk-squared-error}
    Under Assumption~\ref{assump:asymp-cgmt-general}, with probability 1, in the limit as $d \to \infty$ for $\hat{f}(\vx) = \phi(\vx, \widehat{\vbeta}(\setD, \ell, \lambda))$ solving \eqref{eq:erm-ridge},
    for $\psi(z^*, \hat{z}) = (z^* - \hat{z})^2$, there exists a monotonic relation between $\risk_Q(\hat{f})$ and $\risk_P(\hat{f})$ that depends only on $(P, Q, \vbeta^*)$ if and only if
    there exists $\rho > 0$ such that for all $b > 0$,
    \begin{align}
        \Gamma_Q(b) = \rho \Gamma_P(b), \quad 
        \Lambda_Q(b) = \rho \Lambda_P(b), \quad
        \Theta_Q(b) = \rho \Theta_P(b).
    \end{align}
    If this relation exists, it is
        \begin{align}
            \risk_Q(\hat{f}) = \rho ( \risk_P(\hat{f}) - \Omega_P) + \Omega_Q.
        \end{align}
\end{theorem}
\begin{proof}
We begin by observing that 
\begin{align}
    \risk_P(\hat{f}) = \expect{(Z_P^* - \widehat{Z}_P)^2} = \expect{Z_P^{*2}} - 2 \expect{Z_P^* \widehat{Z}_P} +  \expect{\widehat{Z}_P^2},
\end{align}
which means that we can apply Lemma~\ref{lem:linear-risk-monotonic} with $h(t) = t$, $\vw(a, c) = (-2a, a^2, c)$,
and
\begin{align}
    \vv_P(b) = (\Gamma_P(b), \Lambda_P(b), \Theta_P(b)), \quad v_P^0(b) = \Omega_P, \\
    \vv_Q(b) = (\Gamma_Q(b), \Lambda_Q(b), \Theta_Q(b)), \quad v_Q^0(b) = \Omega_Q.
\end{align}
Therefore, the condition that $\vv_Q(b) = \rho \vv_P(b)$ is equivalent to the stated condition. The condition that $v_Q^0(b) = \rho v_P^0(b) + u_0$ is trivially satisfied by $u_0 = v_Q^0(b) - \rho v_P^0(b)$ since $v_P^0$ and $v_Q^0$ are constant functions of $b$.
\end{proof}

\subsection{\ref{proof-step:specific-metrics}: Misclassification Error}

We now move to the slightly more difficult case of misclassification error. We first need a closed-form  expression for the risk, which we obtain from the following lemma.

\begin{lemma}
\label{lem:gaussian-cosine}
For two zero-mean jointly Gaussian random variables $X$ and $Y$, 
\begin{align}
    \Pr(XY < 0) = \frac{1}{\pi} \arccos \paren{\frac{\expect{XY}}{\sqrt{\expect{X^2} \expect{Y^2}}}}.
\end{align}
\end{lemma}
\begin{proof}
First define $\widetilde{X} = X / \sqrt{\expect{X^2}}$ and $\widetilde{Y} = Y / \sqrt{\expect{Y^2}}$.
We can decompose $\widetilde{Y}$ as:
\begin{align}
    \widetilde{Y} = \expect{\widetilde{X}\widetilde{Y}} \widetilde{X} + \sqrt{1 - \expect{\widehat{X}\widetilde{Y}}^2} U_Y,
\end{align}
where $U_Y$ is a standard normal random variable. Observe that for any scalar $a > 0$, $\set{\widetilde{X} \widetilde{Y} < 0} = \set{a\widetilde{X} \widetilde{Y} < 0}$, so we can jointly scale $\widetilde{X}$ and $U_Y$ without affecting the event, even if this scalar is random. Because $\widetilde{X}$ and $U_Y$ are independent standard normal variables, this means we can choose a random variable $\Theta \sim \mathrm{Uniform}[0, 2\pi)$ such that
\begin{align}
    (\cos \Theta, \sin \Theta) = \paren{\frac{\widetilde{X}}{\sqrt{\widetilde{X}^2 + U_Y^2}}, \frac{U_Y}{\sqrt{\widetilde{X}^2 + U_Y^2}}}.
\end{align}
Now 
\[
\Pr(XY < 0) = \Pr \paren{\widetilde{X}\widetilde{Y} < 0} = \Pr \paren{\cos \Theta  \paren{\expect{\widetilde{X}\widetilde{Y}} \cos \Theta + \sqrt{1 - \expect{\widehat{X}\widetilde{Y}}^2} \sin \Theta} < 0}.
\]
This inequality is satisfied for
\begin{align}
    \Theta \in [0, 2 \pi) \cap \bigcup_{n = -\infty}^{\infty} \paren{\frac{(2n + 1)\pi}{2}, \frac{(2n + 1)\pi}{2} + \arccos \paren{\expect{\widetilde{X}\widetilde{Y}}}}.
\end{align}
The size of each of the intervals in the union is $\arccos\paren{\expect{\widetilde{X}\widetilde{Y}}}$, and twice the length of one such interval is included in $[0, 2\pi)$. Plugging in the definitions of $\widetilde{X}$ and $\widetilde{Y}$ therefore proves the claim.
\end{proof}
We are now ready to state and prove the classification error case.

\begin{theorem}
    \label{thm:monotonic-risk-classification-error}
    Under Assumption~\ref{assump:asymp-cgmt-general}, with probability 1, in the limit as $d \to \infty$ for $\hat{f}(\vx) = \phi(\vx, \widehat{\vbeta}(\setD, \ell, \lambda))$ solving \eqref{eq:erm-ridge},
    for $\psi(z^*, \hat{z}) = \ind\set{z^*\hat{z} < 0}$, there exists a monotonic relation between $\risk_Q(\hat{f})$ and $\risk_P(\hat{f})$ that depends only on $(P, Q, \vbeta^*)$ if and only there exist $\rho >0$ and $u_0 \in \reals$ such that \editedinline{for all $b > 0$}
    \begin{align}
        \frac{\Omega_Q \Theta_Q(b)}{\Gamma_Q(b)^2}
        = \frac{\rho\Omega_P \Theta_P(b)}{\Gamma_P(b)^2}, \qquad
        \frac{\Omega_Q \Lambda_Q(b)}{\Gamma_Q(b)^2}
        = \frac{\rho \Omega_P \Lambda_P(b)}{\Gamma_P(b)^2} + u_0.
    \end{align} 
    If this relation exists, it is
    \begin{align}
        \sec^2 (\pi \risk_Q(\hat{f})) = \rho \sec^2(\pi \risk_P(\hat{f})) + u_0,
    \end{align}
    where $\sec(t) = \tfrac{1}{\cos(t)}$.
\end{theorem}
\begin{proof}
Let $h$ have inverse $h^{-1}(t) = \sec^2(\pi t)$. Then applying Lemma~\ref{lem:gaussian-cosine} and the definitions in \eqref{eq:brief-notation}, the risk has the form 
\begin{align}
    h^{-1}(\risk_P(\hat{f})) = \frac{\expect{Z_P^{*2}} \expect{\widehat{Z}_P^2}}{\expect{Z_P^* \widehat{Z}_P}^2} = \Omega_P \frac{a^2\Lambda_P(b) + c \Theta_P(b)}{(a \Gamma_P(b))^2},
\end{align}
which means that we can apply Lemma~\ref{lem:linear-risk-monotonic} with $w(a, c) = \tfrac{c}{a^2}$ and
\begin{gather}
    v_P(b) = \frac{\Omega_P \Theta_P(b)}{\Gamma_P(b)^2}, \quad
    v_Q(b) = \frac{\Omega_Q \Theta_Q(b)}{\Gamma_Q(b)^2}, \quad
    v_P^0(b) = \frac{\Omega_P \Lambda_P(b)}{\Gamma_P(b)^2}, \quad
    v_Q^0(b) = \frac{\Omega_Q \Lambda_Q(b)}{\Gamma_Q(b)^2}.
\end{gather}
The condition from Lemma~\ref{lem:linear-risk-monotonic} is equivalent to the stated condition.
\end{proof}

\subsection{\ref{proof-step:simplifying}: Simplifying Assumptions}

The necessary and sufficient conditions in Theorems~\ref{thm:monotonic-risk-squared-error} and \ref{thm:monotonic-risk-classification-error} are rather difficult to interpret, and they do not simplify cleanly without additional assumptions. The strongest assumption we make is that $\mSigma_P = \mPi_P$ is a projection operator. The advantage of this is that it only has eigenvalues $0$ and $1$, which means that any term involving $(\mSigma_P + b \mI_d)^{-1}$ can have $\tfrac{1}{1 + b}$ factored out, allowing all of the terms to simplify greatly. Because $\vbeta^*$ has i.i.d.\ sub-Gaussian elements, we can without loss of generality assume it to be Gaussian having the same second moment and thus rotationally invariant. Combining these with the simultaneous diagonizability of $\mSigma_Q$ and $\mSigma_P$ gives us the following simplifications:
\begin{gather}
    \Omega_P = r_P \sigma_\beta^2, \quad 
    \Gamma_P(b) = \frac{r_P \sigma_\beta^2}{1 + b}, \quad
    \Lambda_P(b) = \frac{r_P \sigma_\beta^2}{(1 + b)^2}, \quad
    \Theta_P(b) = \frac{r_P}{(1 + b)^2}, \\
    \Gamma_Q(b) = \lim_{d \to \infty}
    \frac{\vbeta_P^{*\transp} \mSigma_Q \vbeta_P^*}{d (1 + b)}, \quad
    \Lambda_Q(b) = \lim_{d \to \infty} \frac{\vbeta_P^{*\transp} \mSigma_Q \vbeta_P^*}{d (1 + b)^2}, \quad
    \Theta_Q(b) = \lim_{d \to \infty} 
    \frac{\tr \bracket{\mSigma_Q \mPi_P}}{d (1 + b)^2}.
\end{gather}
For $\gamma$, $\kappa$, and $\mu$ from Assumption~\ref{assump:asymp-simple}, we therefore have the following relations:
\begin{gather}
    \Gamma_Q(b) = \gamma \Gamma_P(b), \quad
    \Lambda_Q(b) = \gamma \Lambda_P(b), \quad
    \Theta_Q(b) = \kappa \Theta_P(b), \quad \Omega_Q = \gamma \mu \Omega_P, \\
    \frac{\Omega_Q \Theta_Q(b)}{\Gamma_Q(b)^2} = \frac{\mu \kappa \Omega_P \Theta_P(b)}{\gamma \Gamma_P(b)^2}, \quad
    \frac{\Omega_Q \Lambda_Q(b)}{\Gamma_Q(b)^2} = \frac{\mu \Omega_P \Lambda_P(b)}{ \Gamma_P(b)^2} = \mu = \frac{\mu \kappa}{\gamma} + \mu \paren{1 - \frac{\kappa}{\gamma}}.
\end{gather}
For regression, this means that in Theorem~\ref{thm:monotonic-risk-squared-error}, $\rho = \gamma = \kappa$, and $\Omega_Q - \rho \Omega_P = \gamma r_P \sigma_\beta^2 (\mu - 1)$. For classification, this means that in Theorem~\ref{thm:monotonic-risk-classification-error}, $\rho = \tfrac{\mu \kappa}{\gamma}$ and $u_0 = \mu (1 - \tfrac{\kappa}{\gamma})$. These values give the stated claims in Theorem~\ref{thm:monotonic-risks}, and when specializing to $\mu = 1$ for classification, the relation follows by the fact that $\tan^2(\theta) = \sec^2(\theta) - 1$.

\subsection{General Regularization Penalties}
\label{sec:extend-penalty}

The above approach can be used to analyze general separable regularization penalties as well via linearization if $\mSigma_P$ is axis-aligned (that is, diagonal).
Under Assumption~\ref{assump:asymp-cgmt}, the equations in Lemma~\ref{lem:asymp-cgmt-cov} simplify to the forms shown in \ref{proof-step:simplifying} of the proof of Theorem~\ref{thm:monotonic-risks}. Upon closer inspection, we observe that instead of three free variables $a, b, c$, we now only have two degrees of freedom via $\tfrac{a}{1 + b}$ and $\tfrac{c}{1 + b}$. Meanwhile, let
\begin{align}
    \widehat{\vbeta}(\setD, \ell, \lambda) = \argmin_\vbeta \sum_{i=1}^n \ell(y_i, \vx_i^\transp \vbeta) + \lambda \sum_{j=1}^d r([\vbeta]_j)^2
\end{align}
for a convex regularization penalty $r \colon \reals \to \reals$.
Corollary~\ref{cor:asymp-cgmt} can be extended to general regularization penalties (see \citealp{loureiro21}), and our resulting estimator has the following form for each $j \in [d]$ such that $[\mSigma_P]_{jj} = 1$:
\begin{align}
    [\widehat{\vbeta}]_j \simeq \prox{r(\cdot)/b}{\tfrac{a}{b} [\vbeta^*]_j + \tfrac{\sqrt{c}}{b} [\vg]_j}
\end{align}
for some three parameters $a \in \reals$, $b, c > 0$. Assuming $r(u)$ is an increasing function of $|u|$, this implies that the remaining coefficients $\widehat{\vbeta}$ will be 0. 

As in the proof of Lemma~\ref{lem:asymp-cgmt-cov}, we only need to determine the following inner products:
\begin{align}
    \tfrac{1}{d} \vbeta^{*\transp} \mPi_P \widehat{\vbeta}, \quad
    \tfrac{1}{d} \widehat{\vbeta}^\transp \mPi_P \widehat{\vbeta}, \quad
    \tfrac{1}{d} \vbeta^{*\transp} \mSigma_Q \widehat{\vbeta}, \quad
    \tfrac{1}{d} \widehat{\vbeta}^\transp \mSigma_Q \widehat{\vbeta}.
\end{align}
For any $a \in \reals$, $b, c > 0$, we can linearize $\widehat{\vbeta}$ in the form $a' \vbeta^* + \sqrt{c'} \vg$ with respect to $\vbeta^*$ and $\mPi_P$ in the sense that we can find $a' \in \reals, c' > 0$ such that
\begin{align}
    \tfrac{1}{d} \vbeta^{*\transp} \mPi_P \widehat{\vbeta} \asconv a' r_P \sigma_\beta^2, \quad
    \tfrac{1}{d} \widehat{\vbeta}^\transp \mPi_P \widehat{\vbeta}
    \asconv {a'}^2 r_P \sigma_\beta^2 + c' r_P,
\end{align}
which is the same as we have in the ridge regularization case. Therefore, Theorem~\ref{thm:monotonic-risks} will also apply for an arbitrary separable regularizer if and only if 
\begin{gather}
    \tfrac{1}{d} \vbeta^{*\transp} \mSigma_Q \widehat{\vbeta}
    \asconv a' \tfrac{1}{d} \vbeta_P^{*\transp} \mSigma_Q \vbeta_P^*, \\
    \tfrac{1}{d} \widehat{\vbeta}^\transp \mSigma_Q \widehat{\vbeta}
    \asconv {a'}^2 \vbeta_P^{*\transp} \mSigma_Q \vbeta_P^* + c' \tfrac{1}{d} \tr [\mSigma_Q \mPi_P].
\end{gather}
Due to the nonlinearity of the proximal operator, we would not expect these to hold in general, as our linearization only holds for $\widehat{\vbeta}$ measured with respect to $\mPi_P$. However, for example, if $\kappa = \gamma$, then the linearization holds and we can apply Theorem~\ref{thm:monotonic-risks}.

\section{Proof of Theorem~\ref{thm:l2_risk_relationship_in_denoising} and Theorem~\ref{thm:l2_risk_relationship_in_compressed_sensing} }

In this section, we prove the main results on linear relations for linear inverse problems.

\subsection{Proof of Theorem~\ref{thm:l2_risk_relationship_in_denoising}}
\label{sec:proof_of_l2_risk_relationship_in_denoising}

The proof is similar to that of Theorem~\ref{thm:sufficient_necessary_conditions_for_linear_shift_in_regression}. We first express the risk of the signal estimate $\widehat{\vx}_\lambda$ on distribution $P$ as a function of $\alpha = 1/(1 + \sigma_P^2 + \lambda)$. Note that the solution $\mW^*$ to $\min_{\mW} \EX[P]{\norm[2]{\vx - \mW \vy}^2} + \lambda \norm[F]{\mW}^2$ can be computed as
\begin{align}
    \mW^* 
    &= \EX[P]{\vx \vy^\T} \left( \EX[P]{\vy \vy^\T} + \lambda \mI \right)^{-1} \\
    &\stackrel{(a)}{=} \EX[P]{\vx \vx^\T} \left( \EX[P]{\vx \vx^\T + \vz \vz^\T} + \lambda \mI \right)^{-1} \\
    &\stackrel{(b)}{=} \mU_P \mU_P^\T \left( \mU_P \mU_P^\T + (\sigma_P^2 + \lambda) \mI \right)^{-1} \\
    &= \frac{1}{1 + \sigma_P^2 + \lambda} \mU_P \mU_P^\T,
\end{align}
where $(a)$ follows from that $\vz$ is independent of $\vx$ and that $\EX[]{\vz} = 0$, and $(b)$ follows from the assumptions that \editedinline{$\EX[P]{\vc_P \vc_P^\T} = \mI$} and that $\EX[P]{\vz\vz^\T} = \sigma_P^2 \mI$.
Hence, $\widehat{\vx}_\lambda(\vy) = \alpha \mU_P \mU_P^\T \vy$. It holds that
\begin{edited}
\begin{align}
    \risk_P(\widehat{\vx}_\lambda) 
    &= \EX[P]{\norm[2]{(\mU_P \vc_P - \widehat{\vx}_\lambda) / \sqrt{d_P}}^2} \\
    &= \EX[P]{\norm[2]{\left( \mI -  \alpha \mU_P \mU_P^\T \right) \mU_P \vc_P / \sqrt{d_P}}^2} +
    \EX[P]{\norm[2]{\alpha \mU_P \mU_P^\T \vz / \sqrt{d_P}}^2} \\
    &= \EX[P]{\norm[2]{\left( 1 -  \alpha \right) \mU_P \vc_P}^2 / d_P} + 
    \tr\left( \alpha^2 \mU_P \mU_P^\T \EX[P]{\vz\vz^\T} / d_P \right) \\
    &= \tr\left[ \left( 1 -  \alpha \right)^2 \mU_P^\T \mU_P\, \EX[P]{\vc_P \vc_P^\T} / d_P \right] + \alpha^2 \sigma_P^2 \\
    &= (1 - \alpha)^2 + \alpha^2 \sigma_P^2,
\end{align}
\end{edited}
where we have used the assumptions that \editedinline{$\EX[P]{\vc_P \vc_P^\T} = \mI$} and that $\EX[P]{\vz\vz^\T} = \sigma_P^2 \mI$ again.
Similarly, on distribution $Q$,
\begin{edited}
\begin{align}
    \risk_Q(\widehat{\vx}_\lambda) 
    &= \EX[Q]{\norm[2]{(\mU_P\vc_Q - \widehat{\vx}_\lambda) / \sqrt{d_Q}}^2} \\
    &= \EX[Q]{\norm[2]{\left( \mI -  \alpha \mU_P \mU_P^\T \right) \mU_Q \vc_Q / \sqrt{d_Q}}^2} +
    \EX[Q]{\norm[2]{\alpha \mU_P \mU_P^\T \vz / \sqrt{d_Q}}^2},
\end{align}
\end{edited}
where the second term in the line above can be readily found to be $\alpha^2 \sigma_Q^2 d_P/d_Q$, and the first term can be computed as
\begin{edited}
\begin{align}
    \EX[Q]{\norm[2]{\left( \mI -  \alpha \mU_P \mU_P^\T \right) \mU_Q \vc_Q / \sqrt{d_Q}}^2}
    &= \tr\left[ \mU_Q^\T \left( \mI + \left( \alpha^2 - 2\alpha \right)\, \mU_P \mU_P^\T \right) \mU_Q \EX[Q]{\vc_Q \vc_Q^\T} / d_Q \right] \\
    &= \tr\left[ \left( \mI + \left( \alpha^2 - 2\alpha \right)\, \mU_P \mU_P^\T \right) \mU_Q \mU_Q^\T / d_Q \right] \\
    &= \tr\left[ \mU_Q \mU_Q^\T / d_Q \right] + \left( \alpha^2 - 2\alpha \right)\, \tr\left[ \mU^\T \mU_Q \left( \mU_P^\T \mU_Q \right)^\T / d_Q \right] \\
    &\stackrel{(c)}{=} 1 + \left( \alpha^2 - 2\alpha \right)\, \frac{1}{d_Q} \sum_{i=1}^{\min\{d_P, d_Q\}} cos^2(\theta_i) \\
    &= 1 + (\alpha^2 - 2\alpha) \frac{\norm[2]{cos(\vtheta)}^2}{d_Q},
\end{align}
\end{edited}
where $(c)$ follows from the fact that the singular values of $\mU_P^\T \mU_Q$ are the cosines of the principal angle $\theta_i, i\in[\min\{d_P, d_Q\}]$ between $\mU_P$ and $\mU_Q$. Hence, 
\begin{align}
    \risk_Q(\widehat{\vx}_\lambda) = 1 + (\alpha^2 - 2\alpha) \frac{\norm[2]{cos(\vtheta)}^2}{d_Q} + 
    \alpha^2 \sigma_Q^2 \frac{d_P}{d_Q}.
\end{align}
The expression of $\risk_P(\widehat{\vx}_\lambda)$ implies that
\begin{align}
    \alpha^2 - 2\alpha = \risk_P(\widehat{\vx}_\lambda) - 1 -
    \alpha^2 \sigma_P^2.
\end{align}
Plugging this expression into the expression of $\risk_Q(\widehat{\vx}_\lambda)$ yields the result. \hfill$\blacksquare$

\subsection{Proof of Theorem~\ref{thm:l2_risk_relationship_in_compressed_sensing}}
\label{sec:proof_of_l2_risk_relationship_in_compressed_sensing}

We first provide two lemmas which are used in the main proof. In the first lemma, the risks $\risk_P(\widehat{\vx}_\lambda)$ and $\risk_Q(\widehat{\vx}_\lambda)$ are expressed in terms of matrices $\mU_P^\T \mU_P$ and $\mU_P^\T \mU_Q$, and their approximations $\mU_P^\T \mA^\T \mA \mU_P$ and $\mU_P^\T \mA^\T \mA \mU_Q$ induced by the random measurement matrix $\mA$.
\begin{lemma}
The risks $\risk_P(\widehat{\vx}_\lambda)$ and $\risk_Q(\widehat{\vx}_\lambda)$ of $\widehat{\vx}_\lambda$ can be expressed as
\begin{align}
    \risk_P(\widehat{\vx}_\lambda)
    &= \left( \norm[F]{\mI - \mS \mU_P^\T \mA^\T \mA \mU_P}^2 + \sigma_P^2 \tr\left( \mS^\T \mS \mU_P^\T \mA^\T \mA \mU_P \right) \right) / d_P, \\
    \risk_Q(\widehat{\vx}_\lambda)
    &= \left( \norm[F]{\mU_P^\T \mU_Q - \mS \mU_P^\T \mA^\T \mA \mU_Q}^2 - \norm[F]{\mU_P^\T \mU_Q }^2 + \norm[F]{\mU_Q}^2 + \sigma_Q^2 \tr\left( \mS^\T \mS \mU_P^\T \mA^\T \mA \mU_P \right) \right) / d_Q, 
\end{align}
where $\mS = \eta \mI - \eta^2 \mU_P^\T \mA^\T \mA \mU_P \left( \mI + \eta \mU_P^\T \mA^\T \mA \mU_P \right)^{-1}$ and $\eta = 1/ (\sigma_P^2 + \lambda)$.
\label{lem:l2_risk_expressions_in_compressed_sensing}
\end{lemma}
\begin{proof}
Similarly to the proof of Theorem~\ref{thm:l2_risk_relationship_in_denoising}, the solution $\mW^*$ to $\min_{\mW} \EX[P]{\norm[2]{\vx - \mW \vy}^2} + \lambda \norm[F]{\mW}^2$ can be computed as
\begin{align}
    \mW^* 
    &= \EX[P]{\vx \vy^\T} \left( \EX[P]{\vy \vy^\T} + \lambda \mI \right)^{-1} \\
    &= \EX[P]{\vx \vx^\T \mA^\T} \left( \EX[P]{\mA \vx \vx^\T \mA^\T + \vz \vz^\T} + \lambda \mI \right)^{-1} \\
    &= \mU_P \mU_P^\T \mA^\T \left( \mA \mU_P \mU_P^\T \mA^\T + (\sigma_P^2 + \lambda) \mI \right)^{-1} \\
    &\stackrel{(a)}{=} \mU_P \mU_P^\T \mA^\T \left( \eta \mI - \eta^2 \mA \mU_P \left( \mI + \eta \mU_P^\T \mA^\T \mA \mU_P \right)^{-1} \mU_P^\T \mA^\T \right) \\
    &= \mU_P \mS \mU_P^\T \mA^\T,
\end{align}
where $(a)$ follows from the matrix inversion lemma and $\eta = 1 / (\sigma_P^2 + \lambda)$. The risk $\risk_P(\widehat{\vx}_\lambda)$ can be computed as
\begin{edited}
\begin{align}
    \risk_P(\widehat{\vx}_\lambda) 
    &= \EX[P]{\norm[2]{(\mU_P \vc_P - \mW^* (\mA \mU_P \vc_P + \vz)) / \sqrt{d_P}}^2} \\
    &= \EX[P]{\norm[2]{\left( \mI - \mW^* \mA \right) \mU_P \vc_P / \sqrt{d_P}}^2} +
    \EX[P]{\norm[2]{\mW^* \vz / \sqrt{d_P}}^2} \\
    &= \Big( \tr\left( \mU_P^\T (\mI - \mW^* \mA)^\T (\mI - \mW^* \mA) \mU_P \EX[P]{\vc_P \vc_P^\T} \right) + \tr\left( \mW^{*\T} \mW \EX[P]{\vz \vz^\T} \right) \Big) / d_P\\
    &= \left( \norm[F]{(\mI - \mW^* \mA) \mU_P}^2 + \sigma_P^2 \tr\left( \mW^{*\T} \mW^* \right) \right) / d_P \\
    &= \left( \norm[F]{\mU_P (\mI - \mS \mU_P^\T \mA^\T \mA \mU_P)}^2 + \sigma_P^2 \tr\left( \mS^\T \mS \mU_P^\T \mA^\T \mA \mU_P \right) \right) / d_P \\
    &= \left( \norm[F]{\mI - \mS \mU_P^\T \mA^\T \mA \mU_P}^2 + \sigma_P^2 \tr\left( \mS^\T \mS \mU_P^\T \mA^\T \mA \mU_P \right) \right) / d_P. \\
\end{align}
\end{edited}
Similarly, the risk $\risk_Q(\widehat{\vx}_\lambda)$ can be computed as
\begin{edited}
\begin{align}
    \risk_Q(\widehat{\vx}_\lambda) 
    &= \EX[Q]{\norm[2]{(\mU_Q \vc_Q - \mW^* (\mA \mU_Q \vc_Q + \vz)) / \sqrt{d_Q}}^2} \\
    &= \EX[Q]{\norm[2]{\left( \mI - \mW^* \mA \right) \mU_Q \vc_Q / \sqrt{d_Q}}^2} +
    \EX[Q]{\norm[2]{\mW^* \vz / \sqrt{d_Q}}^2}, \\
\end{align}
\end{edited}
where the second term in the line above can be readily found to be $\sigma_Q^2 \tr\left( \mS^\T \mS \mU_P^\T \mA^\T \mA \mU_P \right) / d_Q$, and the first term can be computed as
\begin{edited}
\begin{align}
    &\EX[Q]{\norm[2]{\left( \mI - \mW^* \mA \right) \mU_Q \vc_Q / \sqrt{d_Q}}^2} \\
    &= \tr\left( \mU_Q^\T (\mI - \mW^* \mA)^\T (\mI - \mW^* \mA) \mU_Q \EX[Q]{\vc_Q \vc_Q^\T} \right) / d_Q\\
    &= \norm[F]{(\mI - \mW^* \mA) \mU_Q}^2 / d_Q \\
    &= \norm[F]{\mU_Q - \mU_P \mS \mU_P^\T \mA^\T \mA \mU_Q}^2 / d_Q \\
    &= \norm[F]{(\mU_P \mU_P^\T \mU_Q - \mU_P \mS \mU_P^\T \mA^\T \mA \mU_Q) + (\mI - \mU_P \mU_P^\T)\mU_Q}^2 / d_Q \\
    &\stackrel{(b)}{=} \left( \norm[F]{\mU_P \mU_P^\T \mU_Q - \mU_P \mS \mU_P^\T \mA^\T \mA \mU_Q}^2 + \norm[F]{(\mI - \mU_P \mU_P^\T)\mU_Q}^2 \right) / d_Q \\
    &= \left( \norm[F]{\mU_P^\T \mU_Q - \mS \mU_P^\T \mA^\T \mA \mU_Q}^2 + \norm[F]{\mU_Q - \mU_P \mU_P^\T \mU_Q}^2 \right) / d_Q \\
    &= \left( \norm[F]{\mU_P^\T \mU_Q - \mS \mU_P^\T \mA^\T \mA \mU_Q}^2 - \norm[F]{\mU_P^\T \mU_Q }^2 + \norm[F]{\mU_Q}^2 \right) / d_Q, \\
\end{align}
\end{edited}
where $(b)$ follows from the fact that matrices $\mU_P$ and $\mI - \mU_P \mU_P^\T$ are orthogonal under the Frobenius inner product.
\end{proof}

The second lemma below about inner product preservation is used to show that matrices $\mU_P^\T \mA^\T \mA \mU_P$ and $\mU_P^\T \mA^\T \mA \mU_Q$ are close to $\mU_P^\T \mU_P$ and $\mU_P^\T \mU_Q$ element-wise respectively.
\begin{lemma}
Let $\mA \in \reals^{n \times d}$ be a random Gaussian matrix with independent entries drawn from the distribution $\mathcal{N}(0, 1/n)$. For any $\vu, \vv \in \reals^d$ with $\norm[2]{\vu} \le 1, \norm[2]{\vv} \le 1$ and $0 < \epsilon < 1$, with probability at least $1 - 4\exp(-n \epsilon^2/8)$,
\begin{align}
    \left| \left\langle \mA \vu, \mA \vv \right\rangle - \left\langle \vu, \vv \right\rangle \right| \le \epsilon.
\end{align}
\label{lem:inner_product_preservation}
\end{lemma}
\begin{proof}
The proof relies on the result on norm preservation by random projection: for any $0 < t < 1$, it holds that
\begin{align}
    \Pr\left( \left| \frac{\norm[2]{\mA \vu}^2}{\norm[2]{\vu}^2} - 1 \right| \ge t \right) \le 2 e^{-\frac{n t^2}{8}},
\end{align}
which follows from the fact that, for any $\vu \in \reals^d$, the variable $\norm[2]{\mA \vu}^2 / \norm[2]{\vu}^2$ follows the same distribution as $(1/n) \chi^2(n)$ and that a $(1/n) \chi^2(n)$ distribution is sub-exponential with parameters $(2^2/n, 4)$.

Now consider any $\vu, \vv$ with $\norm[2]{\vu} \le 1, \norm[2]{\vv} \le 1$ and $0 < \epsilon < 1$. Using the fact that $\langle \vu, \vv \rangle = (1/4) \left( \norm[2]{\vu + \vv}^2 - \norm[2]{\vu - \vv}^2 \right)$, under the events that the norm squares $\norm[2]{\vu + \vv}^2$ and $\norm[2]{\vu - \vv}^2$ are approximately preserved, which occur with probability at least $1 - 4\exp(-n\epsilon^2/8)$, it holds that
\begin{align}
    \left\langle \mA \vu, \mA \vv \right\rangle
    &= \frac{1}{4} \left( \norm[2]{\vu + \vv}^2 - \norm[2]{\vu - \vv}^2 \right) \\
    &\le \frac{1}{4} \left( (1 + \epsilon) \norm[2]{\vu + \vv}^2 - (1 - \epsilon) \norm[2]{\vu - \vv}^2 \right) \\
    &= \frac{1}{4} \left( 4\langle \vu, \vv \rangle + 2\epsilon \norm[2]{\vu}^2 + 2\epsilon \norm[2]{\vv}^2 \right) \\
    &\le \langle \vu, \vv \rangle + \epsilon,
\end{align}
and that
\begin{align}
    \left\langle \mA \vu, \mA \vv \right\rangle
    \ge \langle \vu, \vv \rangle - \epsilon,
\end{align}
following a similar derivation.
\end{proof}

\subsubsection{Proof of Theorem~\ref{thm:l2_risk_relationship_in_compressed_sensing}}
The proof idea is essentially the same as the proof of Theorem~\ref{thm:l2_risk_relationship_in_denoising}: expressing the risks of the estimate $\widehat{\vx}_\lambda$ on distributions $P$ and $Q$ as functions of $\alpha = 1/(1 + \sigma_P^2 + \lambda)$ and then expressing the risk $\risk_Q(\widehat{\vx}_\lambda)$ in terms of $\risk_P(\widehat{\vx}_\lambda)$. The only technical issue is that $\mW^* \vy$ is only approximately $\alpha \mU_P \mU_P^\T \vy$ due to the random measurement by matrix $\mA$. We show that, under the event that the map $\vu \mapsto \mA \vu$ approximately preserves inner products of interest, as defined below, the risks $\risk_P(\widehat{\vx}_\lambda)$ and $\risk_Q(\widehat{\vx}_\lambda)$ can be expressed respectively as those in the proof of Theorem~\ref{thm:l2_risk_relationship_in_denoising} plus some error terms which converge to zero as the number of measurements $n \to \infty$ with high probability.

\textit{Step 1. Expressing matrices $\mU_P^\T \mA^\T \mA \mU_P$, $\mU_P^\T \mA^\T \mA \mU_Q$ and $\mS$ as perturbed matrices.} \\
For any $0 < \epsilon < 1/d_P$, consider the event
\begin{align}
    \mathcal{E}:\: \left| \left\langle \mA \vu, \mA \vv \right\rangle - \left\langle \vu, \vv \right\rangle \right| \le \epsilon,\quad &\text{for all pairs of columns $(\vu, \vv)$ of $\mU_P$ and} \\
    &\text{for all column $\vu$ of $\mU_P$ and column $\vv$ of $\mU_Q$},
\end{align}
which happens with probability at least $1 - 4(d_P^2 + d_P d_Q) \exp(-n\epsilon^2/8)$ by Lemma~\ref{lem:inner_product_preservation} and the union bound. Under event $\mathcal{E}$, matrices $\mU_P^\T \mA^\T \mA \mU_P$ and $\mU_P^\T \mA^\T \mA \mU_Q$ are perturbed versions of $\mI$ and $\mU_P^\T \mU_Q$, i.e.,  
\begin{align}
    \mU_P^\T \mA^\T \mA \mU_P
    &= \mI + \mE,\quad |\mE_{ij}| \le \epsilon,\quad \forall\, i \in [d_P],\quad \forall\, j \in [d_P], \\
    \mU_P^\T \mA^\T \mA \mU_Q
    &= \mU_P^\T \mU_Q + \mF,\quad |\mF_{ij}| \le \epsilon,\quad \forall\, i \in [d_P],\quad \forall\, j \in [d_Q], 
\end{align}
and, as a result, $\mS$ is a pertubed version of $\alpha \mI$, i.e.,
\begin{align}
    \mS = \alpha \mI + \mG,
\end{align}
where matrix $\mG$ is a polynomial of matrix $\mE$ as defined below, since
\begin{align}
    \mS 
    &= \eta \mI - \eta^2 (\mI + \mE) \left( \mI + \eta (\mI + \mE) \right)^{-1} \\
    &= \eta \mI - \frac{\eta^2}{1 + \eta} (\mI + \mE) \left( \mI + \frac{\eta}{1 + \eta} \mE \right)^{-1} \\
    &\stackrel{(a)}{=} \eta \mI - \frac{\eta^2}{1 + \eta} (\mI + \mE) \sum_{k=0}^\infty \left( -\frac{\eta}{1 + \eta} \mE \right)^k \\
    &= \eta \mI - \frac{\eta^2}{1 + \eta} \left(\mI + \mE + (\mI + \mE) \sum_{k=1}^\infty\left( -\frac{\eta}{1 + \eta} \mE \right)^k \right) \\
    & = \left( \eta - \frac{\eta^2}{1 + \eta} \right) \mI - \frac{\eta^2}{1 + \eta} \left(\mE + (\mI + \mE) \sum_{k=1}^\infty\left( -\frac{\eta}{1 + \eta} \mE \right)^k \right) \\
    &= \alpha \mI - \frac{\alpha^2}{1 - \alpha} \left(\mE + (\mI + \mE) \sum_{k=1}^\infty (-\alpha \mE)^k \right),
\end{align}
with $\mG = - (\alpha^2/(1 - \alpha)) \left(\mE + (\mI + \mE) \sum_{k=1}^\infty (-\alpha \mE)^k \right)$ and $\alpha = 1 / (1 + \sigma_P^2 + \lambda)$. Recall that $\eta = 1 / (\sigma_P^2 + \lambda)$. Step $(a)$ is valid because $(\eta / (1 + \eta)) \mE$ has eigenvalues bounded by $1$, since for $\epsilon < 1/d_P$, $\max_i |\lambda_i((\eta / (1 + \eta)) \mE)| \le |\lambda_i(\mE)| \le ( \sum_i \lambda_i^2(\mE) )^{1/2} = \norm[F]{\mE} \le \epsilon d_P \le 1$.

\textit{Step 2. Expressing risks $\risk_P(\widehat{\vx}_\lambda)$ and $\risk_Q(\widehat{\vx}_\lambda)$ as functions of $\alpha$ and the perturbation matrices.} \\
Under event $\mathcal{E}$, it holds that
\begin{align}
    \risk_P(\widehat{\vx}_\lambda)
    &= \norm[F]{\mI - (\alpha \mI + \mG) (\mI + \mE)}^2 / d_P + \tr\left( (\alpha \mI + \mG)^2 (\mI + \mE) \right) \sigma_P^2 / d_P \\
    &= \norm[F]{(1 - \alpha) \mI - ((\alpha \mI + \mG) \mE + \mG)}^2 / d_P + \tr\left( \alpha^2 \mI + \alpha^2 \mE + (2\alpha \mG + \mG^2) (\mI + \mE) \right) \sigma_P^2 / d_P \\
    &= (1 - \alpha)^2 + \epsilon_{\text{signal, P}} + \alpha^2 \sigma_P^2 + \epsilon_{\text{noise, P}}, 
\end{align}
where $\epsilon_{\text{signal, P}}$ and $\epsilon_{\text{noise, P}}$ are errors induced by the random measurement and are expressed as
\begin{align}
    \epsilon_{\text{signal, P}}
    &= \tr\big( [(\alpha \mI + \mG) \mE + \mG - 2(1 - \alpha) \mI][(\alpha \mI + \mG) \mE + \mG] \big) / d_P, \\
    \epsilon_{\text{noise, P}}
    &= \tr\left( \alpha^2 \mE + (2\alpha \mG + \mG^2) (\mI + \mE) \right) \sigma_P^2 / d_P, 
\end{align}
and that
\begin{align}
    \risk_Q(\widehat{\vx}_\lambda)
    &= \left( \norm[F]{\mU_P^\T \mU_Q - (\alpha \mI + \mG) (\mU_P^\T \mU_Q + \mF)}^2 - \norm[F]{\mU_P^\T \mU_Q }^2 + \norm[F]{\mU_Q}^2 \right) / d_Q \\
    & \quad\quad\quad\quad\quad\quad\quad\quad\quad\quad\quad\quad\quad\quad\quad\quad\quad\quad\quad\quad + \tr\left( (\alpha \mI + \mG)^2 (\mI + \mE) \right) \sigma_Q^2 / d_Q \\
    &= \left( \norm[F]{(1 - \alpha) \mU_P^\T \mU_Q - ((\alpha \mI + \mG) \mF + \mG \mU_P^\T \mU_Q)}^2 - \norm[F]{\mU_P^\T \mU_Q }^2 + \norm[F]{\mU_Q}^2 \right) / d_Q \\ 
    & \quad\quad\quad\quad\quad\quad\quad\quad\quad\quad\quad\quad\quad\quad + \tr\left( \alpha^2 \mI + \alpha^2 \mE + (2\alpha \mG + \mG^2) (\mI + \mE) \right) \sigma_Q^2 / d_Q \\
    &= (\alpha^2 - 2\alpha) \frac{\norm[F]{\mU_P^\T \mU_Q }^2}{d_Q}  + 1 + \epsilon_{\text{signal, Q}} + \alpha^2 \sigma_Q^2 \frac{d_P}{d_Q} + \epsilon_{\text{noise, Q}}, 
\end{align}
where $\epsilon_{\text{signal, Q}}$ and $\epsilon_{\text{noise, Q}}$ are also errors induced by the random measurement and are expressed as
\begin{align}
    \epsilon_{\text{signal, Q}}
    &= \tr\big( [(\alpha \mI + \mG) \mF + \mG \mU_P^\T \mU_Q - 2(1 - \alpha) \mU_P^\T \mU_Q]^\T [(\alpha \mI + \mG) \mF + \mG \mU_P^\T \mU_Q] \big) / d_Q, \\
    \epsilon_{\text{noise, Q}}
    &= \frac{\sigma_Q^2}{\sigma_P^2} \frac{d_P}{d_Q} \epsilon_{\text{noise, P}}.
\end{align}
The expression of $\risk_P(\widehat{\vx}_\lambda)$ implies that
\begin{align}
    \alpha^2 - 2\alpha = \risk_P(\widehat{\vx}_\lambda) - 1 - \alpha^2 \sigma_P^2 - \epsilon_{\text{signal, P}} - \epsilon_{\text{noise, P}}.
\end{align}
Plugging this expression into the expression of $\risk_Q(\widehat{\vx}_\lambda)$ yields
\begin{align}
    \risk_Q(\widehat{\vx}_\lambda) = a \risk_P(\widehat{\vx}_\lambda) + (1 - a) + \alpha^2 \left( \frac{d_P}{d_Q} \sigma_Q^2 - a \sigma_P^2 \right) + \epsilon_{\text{signal, Q}} + \epsilon_{\text{noise, Q}} - a(\epsilon_{\text{signal, P}} + \epsilon_{\text{noise, P}}),
\end{align}
where $a = \norm[F]{\mU_P^\T \mU_Q }^2 / d_Q = \norm[2]{cos(\vtheta)}^2 / d_Q$ and $\vtheta \in \reals^{\min\{d_P, d_Q\}}$ is the principal angles between $\mU_P$ and $\mU_Q$.

\textit{Step 3. Bounding the errors caused by the perturbation matrices.} \\
It remains to show that, under event $\mathcal{E}$, the error term $\epsilon_{\text{signal, Q}} + \epsilon_{\text{noise, Q}} - a(\epsilon_{\text{signal, P}} + \epsilon_{\text{noise, P}})$ is bounded by some constant times $\epsilon$. We show that each error in the error term is $O(\epsilon)$.

With some computation, it is easy to check that each of the errors $\epsilon_{\text{signal, P}}$, $\epsilon_{\text{noise, P}}$ and $\epsilon_{\text{signal, Q}}$ is the trace of a polynomial of the perturbation matrices, i.e., there exist polynomials $p_1, \ldots, p_5$ such that
\begin{align}
    \epsilon_{\text{signal, P}} &= \tr(p_1(\mE)), \\
    \epsilon_{\text{noise, P}} &= \tr(p_2(\mE)), \\
    \epsilon_{\text{signal, Q}} &= \tr(p_3(\mE) \mF \mF^\T) + \tr(p_4(\mE) \mF \mU_Q^\T \mU_P) + \tr(p_5(\mE) \mU_P^\T \mU_Q \mU_Q^\T \mU_P),
\end{align}
and that $p_1$, $p_2$ and $p_5$ have zero-th order terms zeros.
Recall that matrices $\mE$ and $\mF$ have entries bounded by $\epsilon$. Therefore, it holds that $\norm[F]{\mE} \le \epsilon d_P$ and $\norm[F]{\mF} \le \epsilon \sqrt{d_P d_Q}$, and that
\begin{align}
    |\tr(\mE^k)| 
    &= \big| \sum_i \sigma_i^k(\mE) \big| 
    \stackrel{(b)}{\le} \big| \sum_i \sigma_i(\mE) \big| 
    \le \epsilon d_P,\quad \forall k \ge 1,
\end{align}
where $(b)$ follows from the fact that $\max_i |\sigma_i(\mE)| \le \norm[F]{\mE} \le \epsilon d_P$ for $\epsilon < 1/d_P$. As a result, $\epsilon_{\text{signal, P}}$ and $\epsilon_{\text{noise, P}}$ are $O(\epsilon)$. So is $\epsilon_{\text{signal, Q}}$, because each of its term is $O(\epsilon)$. Indeed, for any $k \ge 0$,
\begin{align}
    \tr(\mE^k \mF \mF^\T) 
    &\le \tr^{\frac{1}{2}}(\mE^{2k})\, \norm[F]{\mF \mF^\T}
    \le \tr^{\frac{1}{2}}(\mI)\, \norm[F]{\mF}^2 
    \le \sqrt{d_P}\, \epsilon^2 d_P d_Q, \\
    \tr(\mE^k \mF \mU_Q^\T \mU_P)
    &\le \tr^{\frac{1}{2}}(\mE^{2k})\, \norm[F]{\mF \mU_Q^\T \mU_P}
    \le \tr^{\frac{1}{2}}(\mI)\, \norm[2]{\mU_Q^\T \mU_P} \norm[F]{\mF} 
    \le \sqrt{d_P}\, \epsilon \sqrt{d_P d_Q},
\end{align}
and for any $k \ge 1$,
\begin{align}
    \tr(\mE^k \mU_P^\T \mU_Q \mU_Q^\T \mU_P)
    \le \tr^{\frac{1}{2}}(\mE^{2k})\, \norm[F]{\mU_P^\T \mU_Q \mU_Q^\T \mU_P}
    &\le \tr^{\frac{1}{2}}(\mE^2) \norm[F]{\mU_P^\T \mU_Q}^2 \\
    &\stackrel{(c)}{\le} \epsilon d_P \min\{d_P, d_Q\},
\end{align}
where $(c)$ follows from the fact that $\norm[F]{\mU_P^\T \mU_Q}^2 = \norm[2]{cos(\vtheta)}^2$. We conclude that the error term $\epsilon_{\text{signal, Q}} + \epsilon_{\text{noise, Q}} - a(\epsilon_{\text{signal, P}} + \epsilon_{\text{noise, P}}) = O(\epsilon)$ and the proof is complete. \hfill$\blacksquare$